\documentclass[journal]{IEEEtran}

\ifCLASSINFOpdf

\else

\fi
\usepackage[normalem]{ulem}
\usepackage{cite}
\usepackage{amsmath}
\usepackage{amssymb}
\usepackage{amsfonts}
\usepackage{amsthm}
\usepackage{mathtools}
\usepackage{graphicx}
\usepackage{color}
\usepackage{epstopdf}
\usepackage{caption}
\usepackage{subfigure}
\usepackage[numbers,sort&compress]{natbib}
\usepackage[left=2cm, right=2cm, top=2.54cm]{geometry}

\newcommand{\norm}[1]{\left\lVert#1\right\rVert}
\newcommand{\stkout}[1]{\ifmmode\text{\sout{\ensuremath{#1}}}\else\sout{#1}\fi}

\usepackage{algorithm}
\usepackage{algorithmic}
\usepackage{xcolor,cancel}
\allowdisplaybreaks

\newtheorem{thm}{Theorem}

\newtheorem{lem}{Lemma}

\newtheorem{assum}{Assumption}
\newtheorem{prop}{Proposition}

\begin{document}
	
	\title{Optimal Stochastic Nonconvex Optimization with Bandit Feedback}
	
	\author{Puning Zhao and
		Lifeng Lai\thanks{Puning Zhao and Lifeng Lai are with Department of Electrical and Computer Engineering, University of California, Davis, CA, 95616. Email: \{pnzhao,lflai\}@ucdavis.edu. This work was supported by the National Science Foundation under grants CCF-17-17943, ECCS-17-11468, CNS-18-24553 and CCF-19-08258. }
	}

	\maketitle
	
\begin{abstract}
In this paper, we analyze the continuous armed bandit problems for nonconvex cost functions under certain smoothness and sublevel set assumptions. We first derive an upper bound on the expected cumulative regret of a simple bin splitting method. We then propose an adaptive bin splitting method, which can significantly improve the performance. Furthermore, a minimax lower bound is derived, which shows that our new adaptive method achieves locally minimax optimal expected cumulative regret. 
	\end{abstract}

\section{Introduction}

Multi armed bandit problem \cite{lai1985asymptotically} is an important sequential decision problem with many applications in signal processing and other related fields. In each step, an agent selects an action among a set called decision space, and receives a feedback loss, which is a random variable with an unknown distribution depending on the selected action. After that, the agent decides the action in the next step, based on previous actions and feedbacks. The goal is to minimize the total expected loss over certain time horizon. With this objective, the design of the sequential decision strategy and the corresponding theoretical analysis have been extensively investigated \cite{lai1985asymptotically,lai1987adaptive,auer2002finite,cappe2013kullback,shen2019universal,shahrampour:TSP:17,vural:TSP:19,gan:TSP:20,Liu:TSP:10,Liu:TSP:20}.

Depending on the shape of the decision space and the cost function, which is the expectation of the feedback loss given the action, the decision strategies are crucially different. For the problem with finite action space, the most popular method is Upper Confidence Bound (UCB) \cite{lai1987adaptive,auer2002finite} and its extensions \cite{garivier2011kl,cappe2013kullback}. It has been proved that the UCB method is minimax rate optimal \cite{auer2002finite}. In these works, the objectives are to maximize reward, while our objective is to minimize the cost, thus the corresponding method for our purpose should be Lower Confidence Bound (LCB). For the problem with infinite action space, the problem becomes more challenging. In fact, without additional assumptions, there is no solution that works uniformly well for all decision spaces and cost functions \cite{wang2009algorithms}. Therefore, many existing works focus on problems whose cost functions exhibit some structural
properties, and propose effective schemes that exploit these properties. It is usually assumed that the decision space is a compact set, and that the decision strategies are selected for different types of cost functions. For example, if the cost function is linear in the decision space, i.e. $f(\mathbf{x})=\mathbf{w}^T \mathbf{x}$, in which $\mathbf{w}$ and $\mathbf{x}$ are $d$ dimensional vectors, then the common idea is to estimate those linear coefficients $\mathbf{w}$ \cite{auer2002using,dani2008stochastic}, and then the optimizer $\mathbf{x}^*$ of the cost function can be found on the boundary of decision space according to the estimated $\mathbf{w}$. As another example, if the cost function is convex, then there are usually two types of selection strategies, those based on gradient descent with an estimated gradient \cite{flaxman2004online,shamir2013complexity} and those based on noisy binary search \cite{agarwal2011stochastic,bubeck2017kernel}. Another class of examples assume that the cost function is Lipschitz continuous. In this case, the most popular method is to split the decision space into multiple bins, and then view all states in each bin as one state, such that the problem can be converted to a multi armed bandit problem with a finite number of decisions \cite{kleinberg2005nearly,kleinberg2008multi,bubeck2011lipschitz}. There are also other previous literatures that discuss the multi armed bandit problems with infinite action space, such as \cite{minsker2013estimation,bull2015adaptive}.

Despite that the multi armed bandit problem with an infinite number of decisions has been analyzed for several different types of cost functions, to the best of our knowledge, the previous analysis does not include general nonconvex cost functions. In many applications, such as the hyperparameter tuning of machine learning algorithms and the sequential design of experiments, the real cost function is usually unknown, and it may not be appropriate to assume that the cost function is linear or convex. Therefore, a solution to continuous armed bandit problems with general nonconvex functions is needed.

In this paper, we analyze the continuous armed bandit problem for nonconvex loss functions in general. Our analysis is based mainly on two assumptions, including a smoothness assumption, characterized by a parameter $\alpha$, and a sublevel set assumption, characterized by a parameter $\beta$. A higher $\alpha$ indicates a higher smoothness level, while a higher $\beta$ indicates that the shape of the loss function is more regular and thus the optimal decision is easier to be located. Similar assumptions have been used in \cite{wang2018optimization} and \cite{minsker2013estimation}. These papers analyze the derivative free optimization problems, i.e. pure exploration problems, which try to minimize the cost function at the final step instead of trying to minimize the cumulative regret. Since we need to focus on the cumulative regret for bandit problems, our task is inherently harder than the pure exploration problems, and the previous methods in \cite{wang2018optimization, minsker2013estimation} can no longer be used here. Therefore, we need to put forward new approaches for the bandit problems. In particular, we make the following contributions.

Firstly, we derive an upper bound of the expected cumulative regret of a simple method based on bin splitting. The basic idea of this simple method is to generate some grid points among the decision space, and then the agent selects decisions among only those grid points using the LCB rule. 
This method has been discussed in \cite{slivkins2019introduction}, Chapter 4 and the references therein. We show that the convergence rate of the average regret of this method is not optimal if $\beta>1$, even if the bin sizes are carefully selected. Intuitively, this is because the optimal bin sizes depend on their locations. In the simple bin splitting method, the sizes are the same for all bins, thus this method will inevitably induce some unnecessary loss. The gap between the convergence rate of the average regret of simple bin splitting method and the optimal convergence rate becomes larger in spaces with higher dimensionality.

Secondly, to improve the performance of the simple bin splitting method, we propose and analyze a new method based on adaptive bin splitting. Our new method is motivated by the following observations. If the cost function value $f(\mathbf{x})$ is far away from the optimal value $f(\mathbf{x}^*)$, then it is not necessary to accurately estimate $f$. Therefore, we can use a large bin size. On the contrary, in the region where $f(\mathbf{x})$ is close to $f(\mathbf{x}^*)$, it is necessary to use small bin size to find the optimizer $\mathbf{x}^*$ more accurately. To achieve such an adaptive splitting without knowing $f(\mathbf{x}^*)$, we divide the decision space into bins with finite capacity. When the number of queries in a bin reaches its capacity, it will split to smaller bins. In the region where $f(\mathbf{x})$ is low, there will be more queries, which will make the bin split many times, and then the optimizer $\mathbf{x}^*$ can be found accurately. On the contrary, if $f(\mathbf{x})$ is much higher then the optimal value $f(\mathbf{x}^*)$, then the bin will subject to less splits, and there will be less queries in these bins. Such adaptive selection rule can significantly improve the convergence rate of average regret. 

Finally, we derive the locally minimax lower bound of the expected cumulative regret, which holds for all decision strategies. For each cost function $f_0$, we find a set of functions $\Sigma_{f_0}$ that are sufficiently close to $f$, and then find a lower bound of the average regret such that no method can achieve a better bound for all cost functions in $\Sigma_{f_0}$. The definition of the locally minimax lower bound shares similar idea with \cite{wang2018optimization}. The result shows that the proposed adaptive splitting method is locally minimax rate optimal up to a logarithmic factor. 

In addition, our method has some additional desirable properties. First, even though our method is designed for nonconvex loss functions, we observe that our method is also minimax optimal for strong convex functions. Second, even though our method is designed to minimize the bandit feedback, it also has optimal optimization error, which means that our new method is also competitive for pure exploration problems. 


\section{Problem Formulation}\label{sec:formulation}
Suppose there is an unknown function $f(\mathbf{x})$ with $\mathbf{x}$ being a $d$-dimentional vector. The function $f$ has optimizer $\mathbf{x}^*$, such that $f(\mathbf{x})\geq f(\mathbf{x}^*)$ for all $\mathbf{x}$. $\mathbf{x}^*$ may not be unique. The exact location of $\mathbf{x}^*$ is unknown, but we know that $\mathbf{x}^*$ is within a compact set $S\subset \mathbb{R}^d$. Moreover, there exists a constant $A$, such that $f$ is defined on $\cup_{\mathbf{x}} B(\mathbf{x}, A)$, in which $B(\mathbf{x},A)$ is the cube centering at $\mathbf{x}$ with length $A$. This indicates that we can query $f$ slightly beyond $S$, and such assumption is common in many previous literatures about stochastic optimization \cite{shamir2013complexity,jamieson2012query,bach2016highly}.  

We need to make $T$ queries $\mathbf{X}_t$, $t=1,\ldots, T$ sequentially. After each query, we receive a cost
$Y_t=f(\mathbf{X}_t)+W_t$,
in which $W_t$ are i.i.d for all positive integer $t$. For simplicity, we assume that $W_t$ follows standard Gaussian distributions $\mathcal{N}(0,1)$. Our results can be easily generalized to the case in which $W_t$ is a subgaussian random variable. Each query $\mathbf{X}_t$ depends on previous queries and responses $\mathbf{X}_1,Y_1,\ldots, \mathbf{X}_{t-1},Y_{t-1}$, i.e.
$\mathbf{X}_t=\sigma_t(\mathbf{X}_1,Y_1,\ldots,\mathbf{X}_{t-1},Y_{t-1}).$
Define the expected cumulative regret $R$ up to time $T$ as
\begin{eqnarray}
R=\mathbb{E}\left[\sum_{t=1}^T (f(\mathbf{X}_t)-f(\mathbf{x}^*))\right].
\end{eqnarray}

Given the total number of queries $T$, our goal is to design a query strategy $\sigma_t$, $t=1,\ldots, T$ to make $R$ as low as possible. Our analysis is based on the following assumptions.

\begin{assum}\label{ass}
	$f$ satisfies the following conditions: there exist constants $A, M, C_0, \alpha, \beta$, such that
	
	(a) $f(\mathbf{x}) - f(\mathbf{x}^*)\leq M$ for all $\mathbf{x}\in S$;
	
	(b) For all $a\in (0, A)$ and all $\mathbf{x}\in S$,
	\begin{eqnarray}
	\left|\frac{1}{V(B(\mathbf{x},a))}\int_{B(\mathbf{x},a)} f(\mathbf{u}) d\mathbf{u}-f(\mathbf{x})\right|\leq Ma^\alpha,
	\label{eq:smooth}
	\end{eqnarray}
	in which $B(\mathbf{x},a)$ is the cube centering at $\mathbf{x}$ with length $a$, and $V(B(\mathbf{x},a))$ denotes its volume;
	
	(c) For all $a>0$ and $\epsilon>0$,
	\begin{eqnarray}
	\mathcal{P}(\{\mathbf{x}|f(\mathbf{x})-f(\mathbf{x}^*)<\epsilon\},a)\leq C_0\left(1+\epsilon^\beta/a^d\right),
	\end{eqnarray}
	in which $\mathcal{P}(S,a)$ is the packing number of $S$ with cubes of length $a$.
\end{assum}
We now comment on these assumptions. In Assumption~\ref{ass} (a), we assume that $f(\mathbf{x})-f(\mathbf{x}^*)$ is bounded above. This assumption is made for the convenience of analysis but is not crucial. If this assumption is violated, we can just use a prescreening step to select a region in which $f(\mathbf{x})-f(\mathbf{x}^*)$ is upper bounded by a constant. There are many methods for prescreening. For example, we can just randomly generate some queries and estimate $f$ using some regression methods \cite{cai2001weighted}, and then calculate the confidence band $(l(\mathbf{x}),u(\mathbf{x}))$ \cite{eubank1993confidence,xia1998bias,neumann1998simultaneous}, such that $l(\mathbf{x})\leq f(\mathbf{x})\leq u(\mathbf{x})$ for all $\mathbf{x}$ holds with high probability. Then let $S'=\{\mathbf{x}|l(\mathbf{x})\leq \sup u(\mathbf{x}) \}$. If $l(\mathbf{x})\leq f(\mathbf{x})\leq u(\mathbf{x})$ for all $\mathbf{x}$, then $\mathbf{x}^*\in S'$. With enough samples, the confidence band can be sufficiently narrow, so that for all $\mathbf{x}\in S'$, $f(\mathbf{x})-f(\mathbf{x}^*)\leq M$. Note that this step will only cause $\mathcal{O}(1)$ regret, thus the final bound of cumulative regret will not change. Such a preprocessing step has been discussed for pure exploration problem \cite{wang2018optimization}. 

Assumption~\ref{ass} (b) is the smoothness assumption. If $f$ is Lipschitz continuous, then (b) is satisfied with $\alpha=1$. If $f$ has bounded Hessian, then (b) holds with $\alpha=2$. Note that $\alpha$ is usually no more than $2$. If Assumption~\ref{ass} (b) holds with some $\alpha>2$, then it can be shown that $f$ is linear. Therefore, in our theoretical analysis later, we assume that $\alpha\leq 2$. Intuitively, this assumption is important because the complexity of finding the optimal point depends highly on the smoothness. If the function is highly smooth, after we find a point $\mathbf{x}$ such that $f(\mathbf{x})$ is far away from $f(\mathbf{x}^*)$, we can make sure that $\mathbf{x}^*$ is far from $\mathbf{x}$, and hence we can query less points around $\mathbf{x}$, and the total complexity can be lower. On the contrary, if the smoothness level is lower, then the complexity must be higher. Assumption (b) can actually be easily generalized to the case such that $f$ is only smooth with parameter $\alpha$ in a neighborhood of the optimal point $\mathbf{x}^*$. If $f$ does not satisfy \eqref{eq:smooth} at some $\mathbf{x}$ such that $f(\mathbf{x})$ is far away from $\mathbf{x}^*$, then the convergence rate of the average regret is not affected.

Assumption~\ref{ass} (c) restricts the volume of the sublevel set. For example, if $f(\mathbf{x})=\norm{\mathbf{x}}_2^2$, or $f$ is any strong convex function, then $\beta=d/2$. If $f(\mathbf{x}) = \norm{\mathbf{x}-\mathbf{c}}$ for a fixed vector $\mathbf{c}$, then $\beta=d$. If $f$ has a higher $\beta$ value, then $f(\mathbf{x})$ is far away to $f(\mathbf{x}^*)$ at the majority of the decision space, and thus we can focus on querying a subspace such that $f(\mathbf{x})$ is close to $f(\mathbf{x}^*)$, and the volume of the subspace is much smaller than that of the whole decision space. As a result, it is easier to find $\mathbf{x}^*$. On the contrary, if $\beta$ is lower, then finding $\mathbf{x}^*$ becomes harder and the regret is higher. Therefore, the convergence rate of the average regret is highly related to the value of $\beta$. This assumption has also been used in \cite{wang2018optimization,minsker2013estimation}. Note that $\beta$ can not be arbitrarily large. Proposition \ref{prop} upper bounds $\beta$ with $d/\alpha$:
\begin{prop}\label{prop}
	If $\alpha\leq 2$, then
	$\beta\leq d/\alpha$.
\end{prop}
\begin{proof}
	Please see Appendix~\ref{app:prop1}.
	\end{proof}

Now we provide several examples satisfying our assumptions. If $y = ||\mathbf{x}-\mathbf{x}^*||^2 + O(||\mathbf{x}-\mathbf{x}^*||^4)$ in the neighborhood of the optimal point $\mathbf{x}^*$, Assumption~\ref{ass} holds with $\alpha=2$ and $\beta=d/2$. If $y=\min\{||\mathbf{x}-\mathbf{c}_1||, ||\mathbf{x}-\mathbf{c}_2||\}$, i.e. there are two optimal points $c_1$ and $c_2$, then Assumption~\ref{ass} holds with $\alpha=1$ and $\beta=d$. If $y=\sin x$, $x\in [a,b]$ such that $b-a$ is sufficiently large, then Assumption~\ref{ass} holds with $\alpha=2$ and $\beta=1/2$. Actually, our assumptions hold for almost all common functions for different $\alpha$ and $\beta$. 

\section{Simple Bin Splitting}\label{sec:simple}

In this section, we derive the convergence rate of the expected average regret of a simple bin splitting method. In the simple bin splitting method, one divides the support into $m$ bins with length $a$, and then convert the problem into a multi armed bandit problem with a finite number of states, which can then be solved using existing methods to design $\sigma_{t}$. The detailed algorithm is shown in Algorithm \ref{alg:base}. Such simple bin splitting method has been analyzed for the case where the cost functions are Lipschitz functions \cite{kleinberg2005nearly}. Our analysis in this section can be viewed as a generalization of \cite{kleinberg2005nearly}, since it is based on a general smooth assumption with $0\leq \alpha\leq 2$, as well as a sublevel set assumption.

\begin{algorithm}[h]
	\caption{Baseline Method}\label{alg:base}
	\begin{algorithmic}
		\INPUT Function $f$, supported on a compact set $S$\\
		\STATE Divide $S$ into $m$ bins with length $a$, called $B_j$, $j=1,\ldots,m$, such that $\mathcal{S}\subset \cup_{j=1}^m B_j$
		\STATE Initialize $n_j=0$, $g_{1j}=-\infty$ for $j=1,\ldots,m$
		\FOR{$t=1,\ldots,T$}
		\STATE $j^*=\underset{j}{\arg\min} g_{tj}$
		\STATE Query $\mathbf{x}_t=\mathbf{c}_{j^*}$, which is the center of $B_{j^*}$, and receive $y_t$
		\STATE $n_{j^*}\leftarrow n_{j^*}+1$
		\FOR{$j=1,\ldots,m$}
		\IF{$n_j=0$}
		\STATE $g_{t+1,j}=-\infty$
		\ELSE
		\STATE $g_{t+1,j}=\frac{1}{n_j}\sum_{t=1}^t \mathbf{1}(\mathbf{x}_i=\mathbf{c}_j)y_i-\sqrt{\frac{8\ln t}{n_j}}$
		\ENDIF
		\ENDFOR
		\ENDFOR
	\end{algorithmic}
\end{algorithm}

In this algorithm, the first step is to divide the support $S$ into $m$ bins with length $a$. All queries are made only on the centers of the bins. For each bin, we use $g_{tj}$ to denote the LCB estimate of the $j$-th bin before the $t$-th query. The initial value $g_{1j}$ is $-\infty$ for all $j$. Each query is selected to be the bin where the LCB estimate $g_{tj}$ is the lowest, and break ties randomly if the bins with minimum LCB values are not unique. After the $t$-th query, the lower confidence bound is then updated to be the mean of the received cost of all queries in each bin. The above process is repeated $T$ times, in which $T$ is the total number of queries.

The expected average regret of Algorithm \ref{alg:base} is bounded by Theorem \ref{thm:base}.

\begin{thm}\label{thm:base}
	If $a\sim T^{-1/(d+2\alpha-\alpha \beta)}$, then the expectation of the average regret $R/T$ is bounded by
	\begin{eqnarray}
	\frac{R}{T}=\mathcal{O}\left( T^{-\frac{\alpha}{d+2\alpha-\alpha \min\{\beta,1\}}}(\ln T)^\frac{1}{2-\min\{\beta,1\}}\right).
	\label{eq:base}
	\end{eqnarray}
\end{thm}
\begin{proof}
	Please refer to Appendix~\ref{app:base}.
\end{proof}


Under the special case of $\alpha=1,\beta=0$, in which the latter condition means that no sublevel set assumption is made on $f$, according to \eqref{eq:base}, the convergence rate of the expected average regret becomes $\mathcal{O}(T^{-\frac{1}{d+2}}\sqrt{\ln T})$, which is consistent with results in \cite{auer2007improved,cope2009regret,kleinberg2005nearly}. Therefore, our result can be viewed as a generalization of those previous bounds. 

We now provide an argument showing that simple bin splitting is not an optimal method. Intuitively, for a larger $\beta$, it is possible that the average regret $R/T$ converges faster to zero, since the volume of the region in which the loss function values are sufficiently close to the optimal value $f(\mathbf{x}^*)$ should be small, and thus the localization of the optimizer $\mathbf{x}^*$ becomes easier. However, from \eqref{eq:base}, we observe that the simple bin splitting method does not make full use of this property. When $\beta>1$, according to \eqref{eq:base}, the convergence rate of the average regret can not be further improved as $\beta$ increases. Therefore, the simple bin splitting method is suboptimal for $\beta>1$. A simple explanation is that in the simple bin splitting method, the sizes of all bins are the same, but the most suitable bin size changes with the location. If we use a small bin size, consider that the LCB value is initially set to be negative infinity, as long as $T>m$, there will be at least one query in each bin. Hence there will be many queries wasted at the locations where $f(\mathbf{x})$ is far from $f(\mathbf{x}^*)$. If we use a large bin size instead, then even if we have found the correct bin that contains $\mathbf{x}^*$, the regret can still be large since the center of the bin is not sufficiently close to $\mathbf{x}^*$. As long as a uniform bin size is used, the sizes will not be universally optimal. Therefore, an adaptive selection rule is needed to further improve the performance.  

\section{Adaptive Bin Splitting}\label{sec:adaptive}
In this section, we propose and analyze a new adaptive bin splitting method, which tries to find the optimal bin size at every location. In particular, we design an approach such that the bins are larger where $f(\mathbf{x})$ is larger than $f(\mathbf{x}^*)$, and smaller otherwise. The main challenge is that $f$ is unknown. In our method, we use the current number of queries in each bin as a measurement of $f(\mathbf{x})-f(\mathbf{x}^*)$. If there are already many queries in a bin, then the cost function values should be close to the optimal value, and thus we need to split the bin to smaller bins in order to better locate the optimizer $\mathbf{x}^*$. The detail of our new method is shown in Algorithm \ref{alg:adaptive}.

\begin{algorithm}[h!]
	\caption{Adaptive Splitting Method}\label{alg:adaptive}
	\begin{algorithmic}
		\INPUT Function $f$, supported on a compact set $S$\\
		\STATE Divide $S$ into $m$ bins with length $a_0$, called $B_j$, $j=1,\ldots,m$, such that $S\subset \cup_{j=1}^m B_j$
		\STATE Initialize $I(B_j)=\emptyset$, $n(B_j)=0$, $k(B_j)=0$, $g_{1}(B_j)=-\infty$ for $j=1,\ldots,m$
		\STATE Initialize $\mathcal{B}=\{B_1,\ldots,B_m \}$
		\FOR{$t=1,\ldots,T$}
		\STATE Pick $B=\underset{B'\in \mathcal{B}}{\arg\min}g_t(B')$
		\IF{$n(B)<\lceil 2^{2\alpha k(B)}\rceil$}
		\STATE Query $\mathbf{x}_t$, which is selected uniformly from $B$, and receive $y_t$
		\STATE Add $t$ to $I(B)$
		\STATE $n(B)\leftarrow n(B)+1$
		\ELSE
		\STATE Split $B$ into $2^d$ bins $B^{(1)},\ldots, B^{(2^d)}$, with each bin having half length comparing with the parent bin
		\STATE $n(B^{(l)})=0$ for $l=1,\ldots, 2^d$
		\STATE $I(B^{(l)})=\emptyset$ for $l=1,\ldots,2^d$
		\STATE $k(B^{(l)})=k(B)+1$ for $l=1,\ldots, 2^d$
		\STATE Add all of these $2^d$ bins to $\mathcal{B}$
		\STATE Query $\mathbf{x}_t$, which is selected uniformly from $B$, and receive $y_t$
		\STATE Find $B^{(l)}$ such that $\mathbf{x}_t\in B^{(l)}$, add $t$ to $I(B^{(l)})$, $n(B^{(l)})\leftarrow n(B^{(l)})+1$
		\STATE Remove $B$ from $\mathcal{B}$
		\ENDIF
		\FOR{$B$ in $\mathcal{B}$}
		\IF{$n(B)>0$}
		\STATE $a=a_02^{-k(B)}$
		\STATE $g_{t+1}(B)=\frac{1}{n(B)}\sum_{i=1}^t \mathbf{1}(i\in I(B))Y_i-\mu a^\alpha-n^{-\frac{1}{2}}(B)\ln t$
		\ELSE
		\STATE $g_{t+1}(B)=-\infty$
		\ENDIF
		\ENDFOR
		\ENDFOR
	\end{algorithmic}
\end{algorithm}

In Algorithm \ref{alg:adaptive}, we begin with dividing $S$ into $m$ bins with length $a_0$, in which $a_0$ is a fixed constant that does not decay with $T$. For each bin $B_j$, denote $I(B_j)$ as the indices of queries in $B_j$, $n(B_j)$ as the number of queries, $k(B_j)$ is the number of splits that $B_j$ already experienced, and $g_t(B_j)$ is the LCB value of bin $B_j$ at time $t$, $t=1,\ldots,T$. Initially, there are no queries, and all bins are initial bins that has not subject to any split, thus $I(B_j)=\emptyset$, $n(B_j)=0$ and $k(B_j)=0$ initially. Besides, at the beginning, the LCB values are all set to be negative infinity. We use $\mathcal{B}$ to denote the set of all bins. $\mathcal{B}$ is dynamic, i.e., when a split happens, the old bin is removed from $\mathcal{B}$, and the new bins generated by splitting the old bin are added to $\mathcal{B}$.

The algorithm then makes $T$ queries. For each query, the algorithm picks a bin with the lowest LCB value. If there are multiple bins with the same LCB value, the algorithm just break ties randomly. For each bin $B$, we set its maximum capacity to be
$n_C(B)=\lceil 2^{2\alpha k(B)}\rceil$,
in which $k(B)$ is the number of splits such that the bin $B$ has already experienced. If the number of queries in the bin does not reach the capacity $n_C(B)$, i.e. $n(B)<n_C(B)$, then we just make the query in the bin $B$ with the lowest LCB value, and then add the time step $t$ to $I(B)$. Unlike the simple splitting method, we let the query $\mathbf{x}_t$ to be taken with uniform probability in $B$, instead of querying the center of $B$. This setting facilitates our subsequent theoretical analysis. If $n(B)=n_C(B)$, then we split this bin into $2^d$ smaller bins, with each bin having half length comparing with the parent bin $B$. These bins are added into $\mathcal{B}$ and the original bin $B$ is removed. For all of these new bins $B^{(l)}$ generated from $B$, in which $l=1,\ldots,2^d$, we initiate them with $n(B^{(l)})=0$, $I(B^{(l)})=\emptyset$ and $k(B^{(l)})=k(B)+1$, which holds because these new bins have experienced one more split than their parent bin $B$.

We set the capacity to be $n(B)=\lceil 2^{2\alpha k(B)}\rceil$ based on the following intuition. The classical analysis of multi armed bandit problem \cite{auer2002finite} shows that the number of queries of each arm scales roughly with $n\sim 1/(f(\mathbf{c}(B))-f(\mathbf{x}^*))^2$, in which $\mathbf{c}(B)$ is the center of the bin. If the number of queries reaches $2^{2\alpha k(B)}$, then $f(\mathbf{x})-f(\mathbf{x}^*)\lesssim 2^{-\alpha k(B)}$ roughly holds for all $\mathbf{x}\in B$. Note that in the bin that contains $\mathbf{x}^*$, denoted as $B^*$, $f(\mathbf{x})-f(\mathbf{x}^*)\lesssim a^\alpha$, in which $a$ is the bin length. After $k$ splits, the bin length should be $a\sim 2^{- k(B)}$, hence $f(\mathbf{x})-f(\mathbf{x}^*)\lesssim 2^{-\alpha k(B^*)}$ roughly holds for $\mathbf{x}\in B^*$. According to the discussions above, it is possible that $B=B^*$, which means that $B$ contains the optimal point. Therefore, in our algorithm, this bin is split to smaller bins in order to help us find $\mathbf{x}^*$ with a higher accuracy. 

Note that $\alpha$ may be unknown in practice. In this case, it would be better to use an $\alpha$ value that is smaller than the true value, which means that we assign each bin with a smaller capacity so that the bins are easier to be split further. In other words, when whether $B$ contains $\mathbf{x}^*$ is unknown because $\alpha$ is unknown, we would rather judge that $B$ contains $\mathbf{x}^*$ and split this bin further. There will inevitably be some drawbacks, since there will be more unnecessary queries in the bin that does not contain $\mathbf{x}^*$. Suppose that we use $\alpha_0$ in Algorithm \ref{alg:adaptive}, in which $\alpha_0<\alpha$, then the convergence rate of the average regret depends on $\alpha_0$ instead of $\alpha$, which is suboptimal. Therefore, Algorithm \ref{alg:adaptive} can not adapt perfectly to different smoothness parameters. However, we can still claim the optimality of our algorithm, since according to Theorem 3 in \cite{locatelli2018adaptivity}, there is no method that has optimal rate universally for all smoothness parameters.

After each query, the algorithm updates the LCB value 
\begin{eqnarray}
g_{t+1}(B)&=&\frac{1}{n(B)}\sum_{i=1}^t \mathbf{1}(i\in I(B))Y_i\nonumber\\
&&-\mu a^\alpha-n^{-\frac{1}{2}}(B)\ln t,
\end{eqnarray}
in which the first term is the average of the received loss for all queries in $B$. This average value can be used as an estimate of the cost function values in $B$. The second term $-\mu a^\alpha$ is the bias correction. When the optimal $\mu$ is unknown, we can just use a large $\mu$. The performance of the algorithm becomes worse if $\mu$ is too large, but the convergence rate of the average regret over $T$ remains unchanged. The third term comes from the noise $W_t$. We construct the lower confidence bound in this way so that with high probability, the LCB value of each bin is lower than the corresponding cost function values. If there is no query now, the LCB value is set to be negative infinity, in order to encourage the exploration in this bin.

The whole process is repeated $T$ times. Theorem \ref{thm:adaptive} provides an upper bound of the average regret of the adaptive bin splitting method.

\begin{thm}\label{thm:adaptive}
	If
	$\mu > (1+2^{d+\alpha})M$,
	then the expectation of the average regret $R/T$ is bounded by
	\begin{eqnarray}
	\frac{R}{T}=\mathcal{O}\left( T^{-\frac{\alpha}{d+2\alpha-\alpha \beta}}\ln^{\beta+1}T\right).
	\label{eq:adaptive}
	\end{eqnarray}
\end{thm}
\begin{proof}
	Please refer to Appendix~\ref{app:adaptive}.
\end{proof}

Comparing Theorem \ref{thm:adaptive} with Theorem \ref{thm:base}, we can observe that if $\beta\leq 1$, \eqref{eq:adaptive} and \eqref{eq:base} are actually the same, except a logarithmic factor. However, when $\beta>1$, these two bounds are different. Unlike the simple bin splitting method, the convergence rate of the average regret of the adaptive binning method continues to improve with the increase of $\beta$, as long as $\beta$ is no more than its maximum value $d/\alpha$. This comparison indicates that the advantage of our new method is more obvious for larger $\beta$. Such result is meaningful especially in high dimensional spaces, since $\beta$ usually increases with $d$. 

Now we discuss the case where $\beta$ reaches its maximum value $d/\alpha$. Note that $\beta=d/\alpha$ is common for practical functions. For example, if $c_1\norm{\mathbf{x}}^p\leq f(\mathbf{x})\leq c_2\norm{\mathbf{x}}^p$ for some constant $c_1$, $c_2$ and $p\in (0,2]$ in a neighbor around $\mathbf{x}^*$, then $f$ satisfies Assumption \ref{ass} with $\alpha=p$ and $\beta=d/p$. In this case, the average regret converges with $\mathcal{O}(T^{-1/2}\ln^{d/\alpha+1} T)$, which only depends on $T$ in its logarithm factor. Hence, for many practical functions satisfying $\beta=d/\alpha$, the convergence rate of the average regret does not decay significantly with the increase of dimensionality. An intuitive explanation is that given $\beta=d/\alpha$, $\beta$ increases with $d$, thus the proportion of the space such that $f(\mathbf{x})-f(\mathbf{x}^*)$ is small becomes lower with the increase of dimensionality. As a result, although the complexity of querying the whole spaces increases with dimensionality, the proportion of the whole decision space that needs further query decreases with the dimensionality, and these two effects tend to cancel out, and therefore the convergence rate of the average regret does not depend significantly on the dimensionality.

\section{Minimax Lower Bound}\label{sec:minimax}

In this section, we show the locally minimax lower bound of the nonconvex optimization problem with bandit feedback under Assumption \ref{ass}, which takes supremum over all functions that is sufficiently close to a reference function $f_0$, and takes infimum for all possible estimators \cite{van2000asymptotic}. 

Our construction of the locally minimax lower bound shares similar ideas with \cite{wang2018optimization}. We assume that the estimator has complete knowledge of a reference function $f_0$. Theorem \ref{thm:mmx} shows that there exists a convergence rate $\underset{\sigma}{\inf}\underset{f\in \Sigma_{f_0}}{\sup} R/T$, such that even if the perfect knowledge of $f_0$ is available, one can not find a decision rule with average regret converging faster than $\underset{\sigma}{\inf}\underset{f\in \Sigma_{f_0}}{\sup} R/T$:

\begin{thm}\label{thm:mmx}
	Assume that $f_0$ satisfies the following conditions:
	
	(a) $f_0(\mathbf{x})-f_0(\mathbf{x}_0^*)\leq M/2$;
	
	(b) For all $a\in (0, A)$ and all $\mathbf{x}\in S$,
	\begin{eqnarray}
	\left|\frac{1}{V(B(\mathbf{x},a))}\int_{B(\mathbf{x},a)} f_0(\mathbf{u}) d\mathbf{u}-f_0(\mathbf{x})\right|\leq \frac{1}{2}Ma^\alpha;
	\end{eqnarray}
	
	(c) There exists a constant $c_0$, such that for sufficiently small $r$,
	$\mathcal{P}\left(\{\mathbf{x}|f_0(\mathbf{x})<\epsilon\},r\right)\geq c_0\frac{\epsilon^\beta}{r^d},$
	in which $\mathcal{P}$ is the packing number. Then there exists a sequence $\epsilon_T$ that decays with $T$, such that
	\begin{eqnarray}
	\underset{\sigma}{\inf}\underset{f\in \Sigma_{f_0}}{\sup}R/T=\Omega\left( T^{-\frac{\alpha}{d+2\alpha-\alpha \beta}}\right),
	\label{eq:minimax}
	\end{eqnarray}
	in which $\Sigma_{f_0}$ is the set of all functions $f$ that satisfy Assumption \ref{ass} and $\norm{f-f_0}_\infty<\epsilon_T$.
\end{thm}
\begin{proof}
Note that minimizing the average regret is always harder than minimizing the optimization error. There is already a proof of the corresponding theorem for optimization error in \cite{wang2018optimization}, which can be used to prove Theorem \ref{thm:mmx} with some minor modification. We omit the detailed proof here for brevity.	
\end{proof}

Theorem \ref{thm:mmx} provides a locally minimax bound, which means that as long as the reference function $f_0$ satisfies some conditions, no method can achieve $T^{-\alpha/(d+2\alpha-\alpha \beta)}$ average regret for all $f$ that is in a close neighborhood of $f_0$, even if we have perfect knowledge of $f_0$.

In Theorem \ref{thm:mmx}, Assumptions (a), (b) are almost the same as those in Assumption \ref{ass}, except that the constant $M$ has been changed to $M/2$. We use this trick to ensure that all functions in the neighbor of $f_0$, i.e. $\Sigma_{f_0}$, satisfy Assumption \ref{ass}. Note that even though $f$ is sufficiently close to $f_0$, these two functions are not exactly the same. If we use $M$ instead, then $f_0$ can be a function that satisfies (a), (b) but close to be violated, and thus it is possible that $f$ violates (a) and (b). To avoid this, we use $M/2$ instead. (c) can be viewed as a reverse of Assumption \ref{ass} (c), which roughly means that Assumption \ref{ass} (c) is tight. This assumption is the same as Assumption (A2') in \cite{wang2018optimization}.

Comparing \eqref{eq:minimax} with \eqref{eq:base} and \eqref{eq:adaptive}, it can be observed that the simple bin splitting method is nearly minimax optimal up to a logarithmic factor only when $\beta\leq 1$. When $\beta>1$, the simple splitting method fails to improve with the increase of $\beta$. On the contrary, our new method is nearly minimax optimal for all possible $\alpha$ and $\beta$. This result indicates that our adaptive splitting method can not be further improved by more than a logarithm factor of $T$.

\section{Comparison with Related Works}\label{sec:compare}


\subsection{Comparison with bandit problem with strongly convex cost functions}
As discussed in Section~\ref{sec:adaptive}, if $\beta$ reaches its maximum $\beta=d/\alpha$, then the average regret converges with $\mathcal{O}\left(T^{-1/2}\ln^{d/\alpha+1} T\right)$. A typical example is that if $f$ is strongly convex and has bounded Hessian, i.e. $A I_d\preccurlyeq \nabla^2 f\preccurlyeq BI_d $, in which $A, B$ are constants and $I_d$ is a $d\times d$ identity matrix, then it is straightforward to show that Assumption \ref{ass} holds with $\alpha=2$ and $\beta=d/2$. In this case, our adaptive bin splitting method achieves the average regret bound $\tilde{\mathcal{O}}(T^{-1/2})$. Previous research \cite{agarwal2011stochastic,shamir2013complexity,hazan2014bandit,bubeck2017kernel} has shown that the bound $\tilde{\mathcal{O}}(T^{-1/2})$ is the best even for methods that are designed specifically for strongly convex cost functions, up to a logarithmic factor. This result indicates that even though our method is designed to solve nonconvex bandit problems, it is also optimal for strongly convex problems.
\subsection{Comparison with pure exploration problems}
Another interesting comparison is with the pure exploration problem discussed in \cite{wang2018optimization,polyak1990optimal,bach2016highly}. Here, we discuss two cases depending on the H{\"o}lder smoothness order.

(1) If the cost function $f$ is H{\"o}lder continuous with order no more than $2$, then Assumption \ref{ass} holds with some $\alpha\leq 2$. In this case, our bound \eqref{eq:adaptive} is exactly the same as Proposition 2 in \cite{wang2018optimization}, except a logarithmic factor. This result indicates that our method is rate optimal not only for bandit problems, but also for pure exploration problems. Moreover, we can show that our method requires less computation cost than the optimization method in \cite{wang2018optimization}. In \cite{wang2018optimization}, a large set of grid points, whose size is much larger than the total number of queries, need to be generated and used for later computation. The total time complexity is high due to this process. On the contrary, our new method has a much lower time complexity. By keeping all bins in a heap with the LCB values as their keys, there will be $\mathcal{O}(T)$ heap operations, with each operation requires $\mathcal{O}(\log |\mathcal{B}|)$ time. According to the capacity formula $n(B)=\lceil 2^{2\alpha k(B)}\rceil$, $|\mathcal{B}|$ grows with $\log T$, hence the total complexity is $\mathcal{O}(T\log\log T)$. This indicates that our new adaptive method reduces the time cost significantly, and thus this method is competitive also for pure exploration problems. 

(2) For highly smooth functions, which are H{\"o}lder continuous with order higher than $2$, our bound no longer matches the pure exploration error in \cite{wang2018optimization}. Such gap is inevitable since the bandit problem is inherently harder than pure exploration. In fact, even for quadratic functions, which are infinitely differentiable, the minimax analysis shows that the average regret can not converge faster than $\mathcal{O}(T^{-1/2})$. For more complex nonconvex functions, $\mathcal{O}(T^{-1/2})$ is the best rate one can expect. However, for pure exploration problems, it is possible to achieve faster rates if the function is highly smooth \cite{polyak1990optimal,bach2016highly}. Combining these two cases, we can see that for functions with low smoothness level, our method is optimal for both bandit and pure exploration problems. For highly smooth functions, no method is optimal for both problems.

Another difference between pure exploration and bandit problems is that for pure exploration problems, it is possible to design a method that is adaptive to smoothness parameters \cite{locatelli2018adaptivity,munos2011optimistic}, while for bandit problems, according to Theorem 3 in \cite{locatelli2018adaptivity}, such adaptivity is impossible. If we are not certain about the smoothness parameter, we can only use an $\alpha$ value that is smaller than the real smoothness parameter, and the convergence rate of the average regret becomes slower than the optimal rate.

\section{Numerical Examples}\label{sec:numerical}

In this section, we provide numerical examples to validate our theoretical analysis. The numerical simulation includes two parts. In particular, we compare the regret of our adaptive method with the simple splitting method.

In the simulation, we use two functions:
$f_1(\mathbf{x})=10 \norm{\mathbf{x}+\mathbf{c}}_2^2,$
and
$f_2(\mathbf{x})=10\min\{\norm{\mathbf{x}-\mathbf{c}}, \norm{\mathbf{x}+\mathbf{c}}\},$
in which $f_1$ as an example of convex function, and $f_2$ is an example of nonconvex function. $d$ is the dimensionality of the support set. We run simulation with $d=1,2,3$ separately. In both $f_1$ and $f_2$, $\mathbf{x}\in [-1,1]^d$.

The construction $f_1$ and $f_2$ can represent many practical cases. For example, in parameter tuning tasks, it is natural to assume that $f$ is strongly convex in a neighborhood of the optimal point $\mathbf{x}^*$, then the convergence of the average regret of this case will be similar to that of $f_1$. The other example $f_2$ can be used to represent the case in which there are multiple optimal actions. 

For each function, we run simulation $100$ times, with each run having $T=10,000$ queries. For the simple bin splitting method, we run simulation for multiple bin sizes, and generate a plot of the regret versus the bin length, as is shown in the blue curves in Figure \ref{fig:compare}. Moreover, we use an orange dashed line to show the regret of the adaptive splitting method. 

\begin{figure}[h!]
	\begin{center}
		\subfigure[$d=1$, cost function $f_1$]{\includegraphics[width=0.49\linewidth, height = 2.8cm]{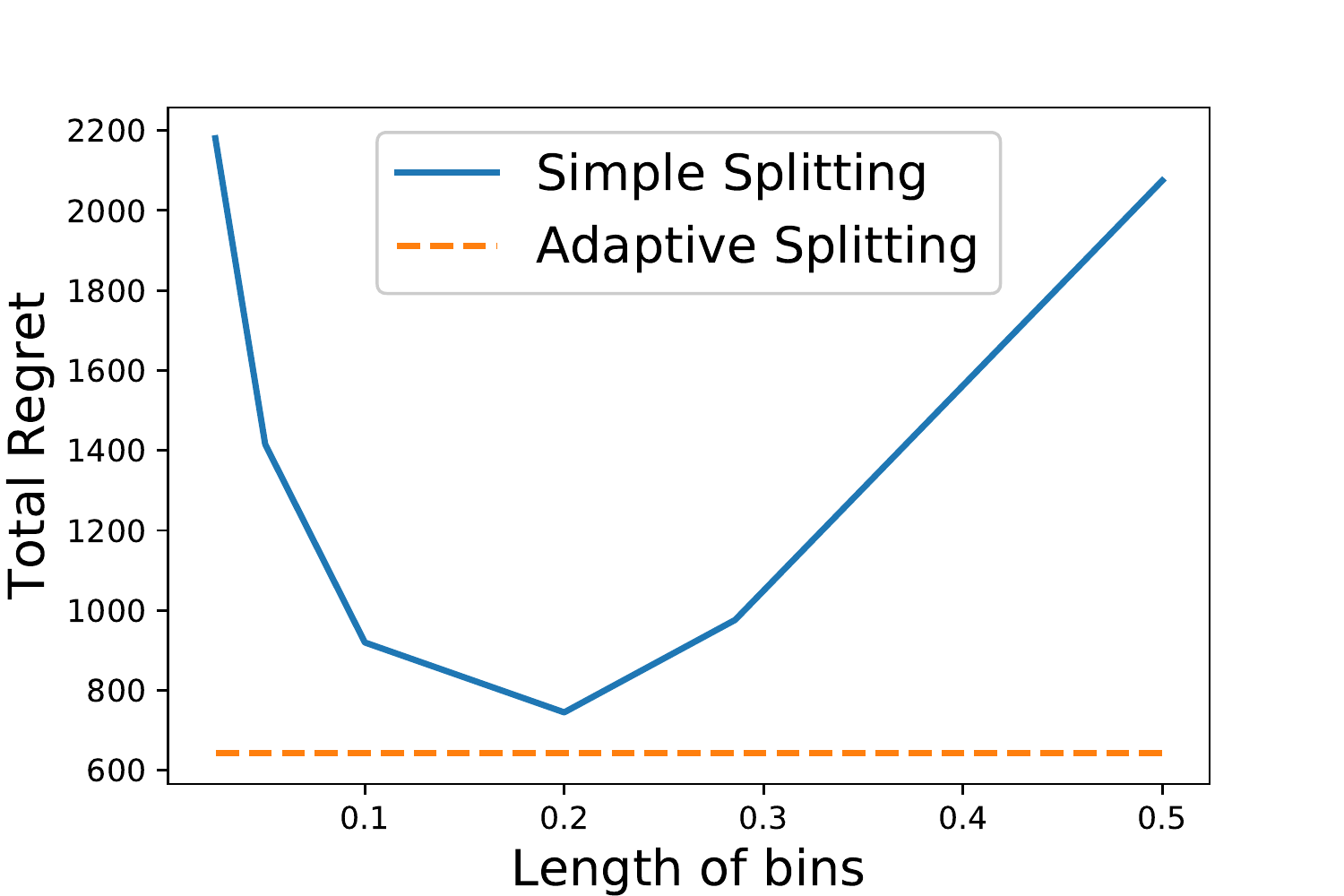}}	
		\subfigure[$d=2$, cost function $f_1$]{\includegraphics[width=0.49\linewidth, height = 2.8cm]{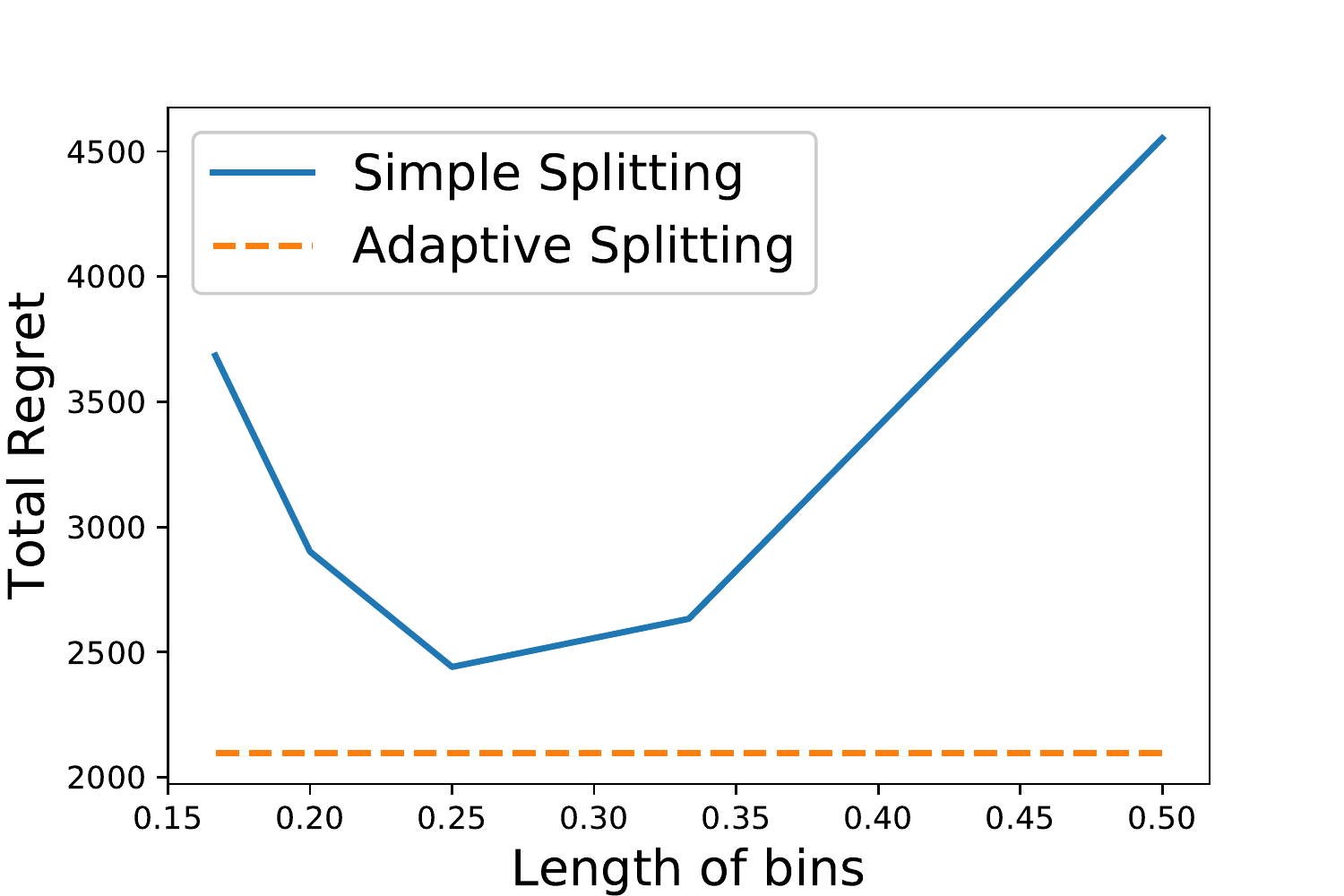}}		
		\subfigure[$d=3$, cost function $f_1$]{\includegraphics[width=0.49\linewidth, height = 2.8cm]{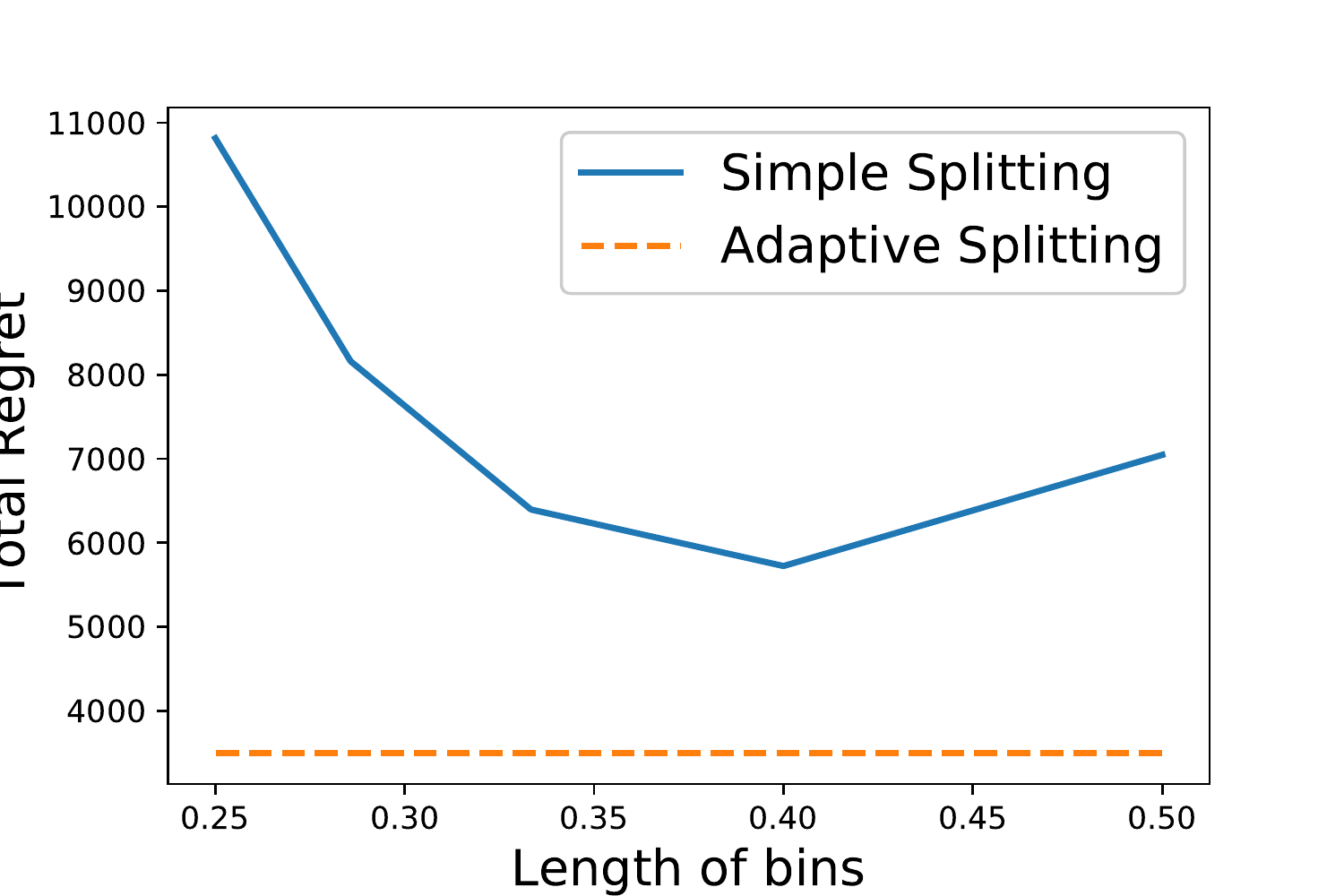}}	
		\subfigure[$d=1$, cost function $f_2$]{\includegraphics[width=0.49\linewidth, height = 2.8cm]{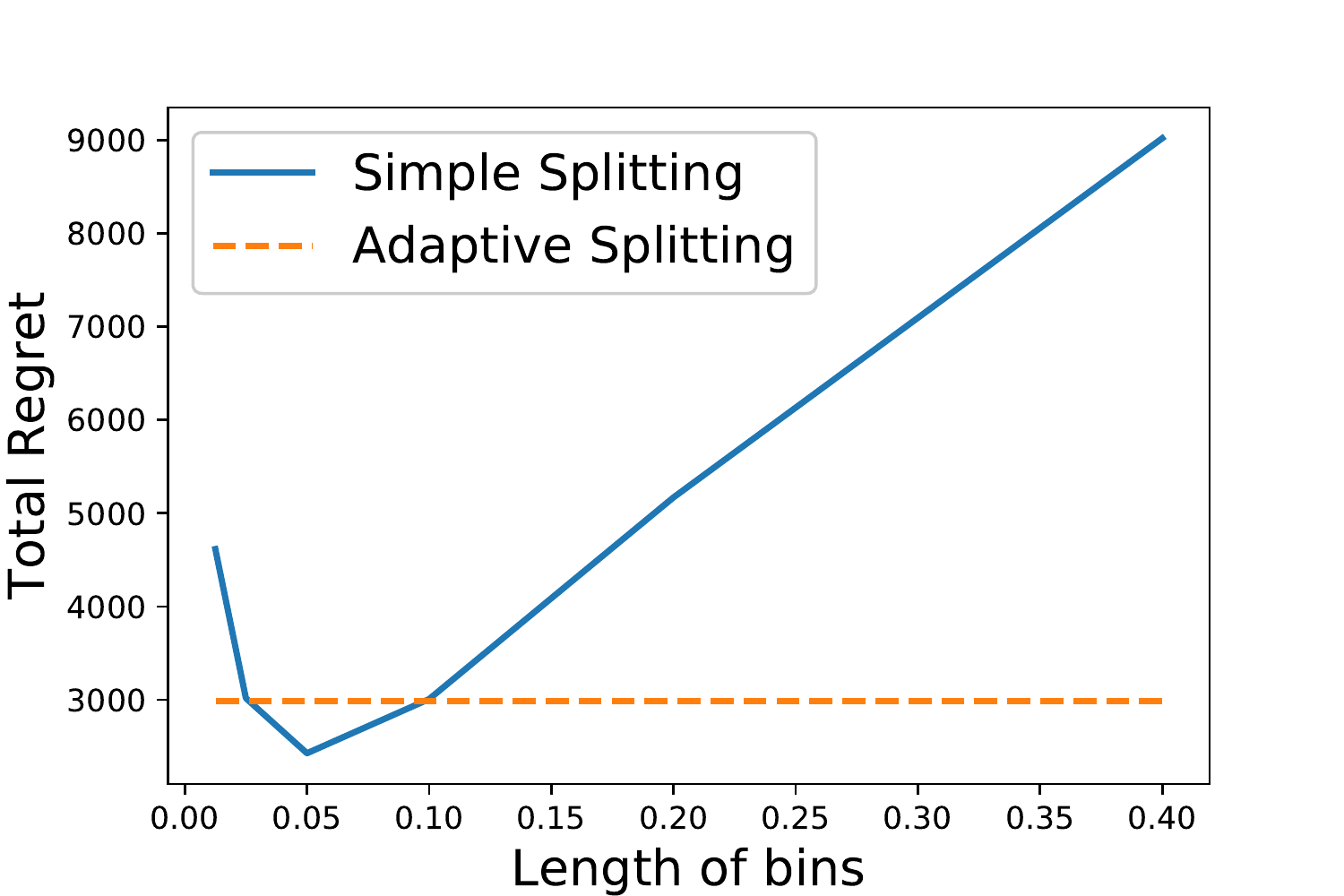}}	
		\subfigure[$d=2$, cost function $f_2$]{\includegraphics[width=0.49\linewidth, height = 2.8cm]{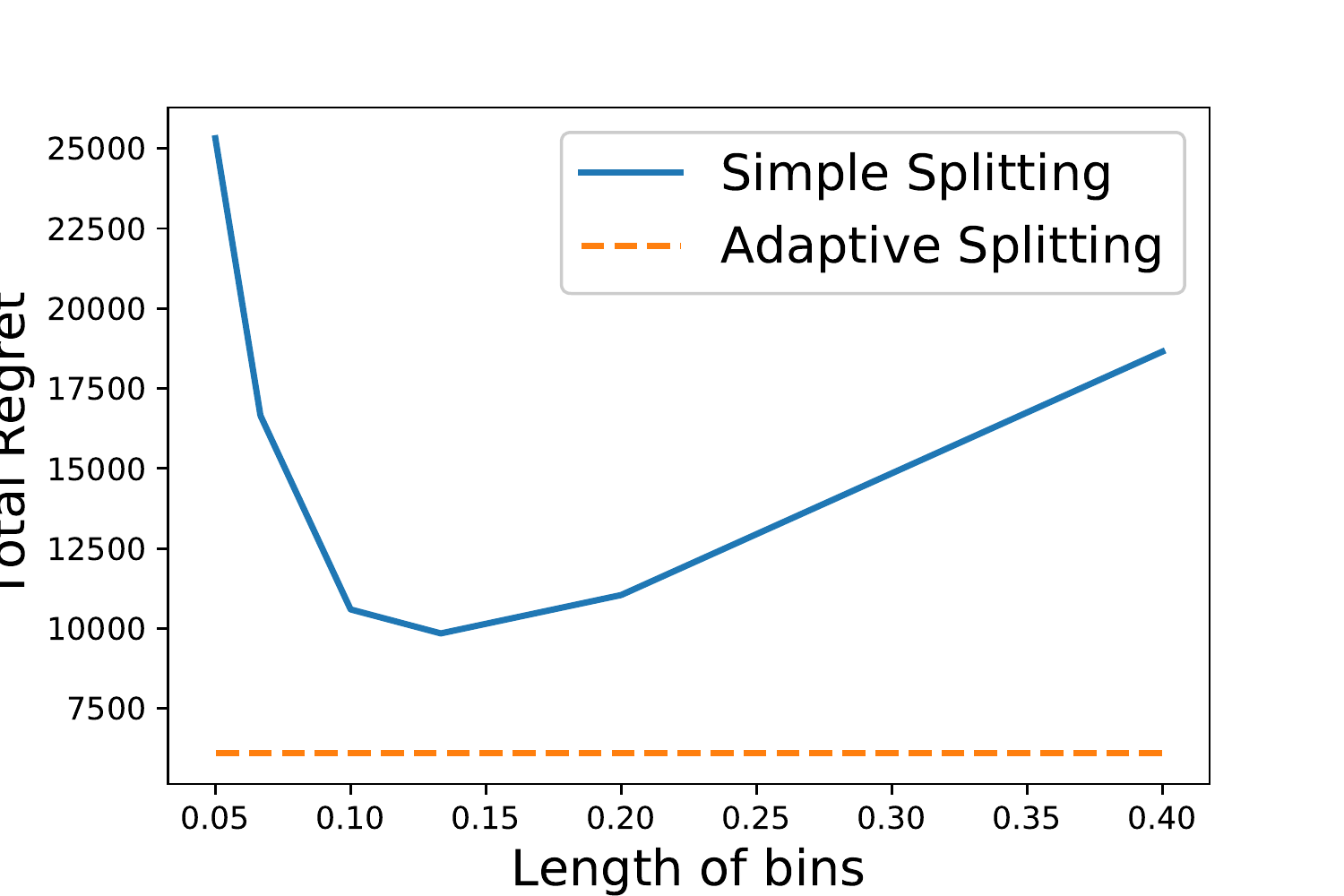}}	
		\subfigure[$d=3$, cost function $f_2$]{\includegraphics[width=0.49\linewidth, height = 2.8cm]{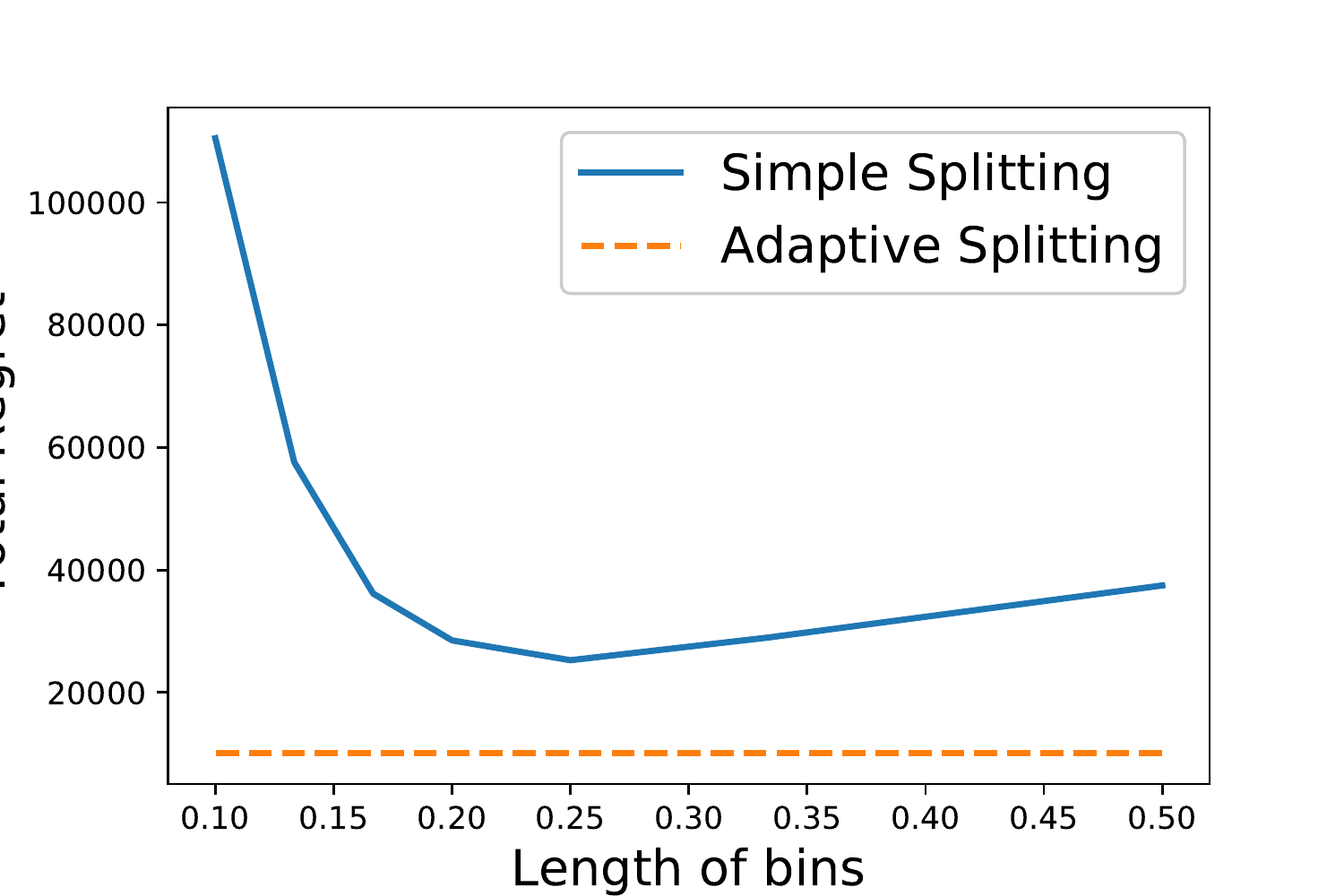}}						
		\caption{Comparison between simple splitting method and the proposed adaptive splitting method. The blue curve is the plot of regret vs bin size for simple splitting method, while the orange dashed line is the regret of the adaptive method. }\label{fig:compare}	
	\end{center}
\end{figure}

According to Figure \ref{fig:compare}, for most cases, the adaptive method is significantly better than the simple splitting method, even if the bin size of the latter is optimally selected. One exception is occurred for $f_2$ and $d=1$. Note that $f_2$ satisfies Assumption \ref{ass} with $\alpha=1$, $\beta=1$. According to Theorem \ref{thm:base} and Theorem \ref{thm:adaptive}, there is no significant difference between theoretical regret bound of these two methods. However, with higher dimensionality, the advantage of our adaptive splitting method becomes more obvious.

\section{Conclusion}\label{sec:conc}
In this paper, we have analyzed the general nonconvex optimization problem with bandit feedback under a smoothness assumption and a sublevel set assumption. We have derived the convergence rate of the average regret of simple bin splitting method. The result shows that the simple splitting method is optimal only for a restricted class of cost functions, and is suboptimal in other cases. 
To address this, we have proposed an adaptive splitting rule and have derived the convergence rate of the average regret. 
We have conducted locally minimax analysis to show that our method is optimal. We have compared our results with previous literatures, and showed that despite our method is designed for nonconvex bandit problems, it is also optimal for strongly convex bandit problems and nonconvex pure exploration problems. Finally, we conducted numerical simulations to validate our theoretical analysis.

\appendices
\section{Proof of Proposition 1}\label{app:prop1}
We begin with the following lemma:
\begin{lem}\label{lem:approx}
	There exists a constant $C_1$, such that for any $\mathbf{x}\in S$,$f(\mathbf{x})-f(\mathbf{x}^*)\leq C_1\norm{\mathbf{x}-\mathbf{x}^*}^\alpha$.
\end{lem}
\begin{proof}
	Define $r_x:=\norm{\mathbf{x}-\mathbf{x}^*}$,
	in which $\norm{\cdot}$ is the sup norm. Denote $B_x$ as the cube centering at $\mathbf{x}$ with half length $r_x$, and $B_x'$ as the cube centering at $\mathbf{x}^*$ with half length $2r_x$. Then according to Assumption (b),
	\begin{eqnarray}
	&&\hspace{-5mm}|f(\mathbf{x})-f(\mathbf{x}^*)|\leq \left|f(\mathbf{x})-\frac{1}{V(B_x)}\int_{B_x}f(\mathbf{u})d\mathbf{u}\right|\nonumber\\
	&&+\left|\frac{1}{V(B_x)}\int_{B_x} f(\mathbf{u})d\mathbf{u}-f(\mathbf{x}^*)\right| \nonumber\\
	&\leq & Mr_x^\alpha + \frac{1}{V(B_x)}\int_{B_x} (f(\mathbf{u}-f(\mathbf{x}^*)))d\mathbf{u}\nonumber\\
	&\leq &Mr_x^\alpha + \frac{1}{V(B_x)}\int_{B_x'}(f(\mathbf{u}-f(\mathbf{x}^*)))d\mathbf{u}\nonumber\\
	&\leq &Mr_x^\alpha +\frac{V(B_x')}{V(B_x)} M(2r_x)^\alpha= (1+2^{d+\alpha})Mr_x^\alpha.\nonumber
	\end{eqnarray}
\end{proof}
With Lemma \ref{lem:approx}, for all $\mathbf{x}$ such that $\norm{\mathbf{x}-\mathbf{x}^*}<(\epsilon/C_1)^{1/\alpha}$, we have $f(\mathbf{x})-f(\mathbf{x}^*)<\epsilon$. Hence
\begin{eqnarray}
\hspace{-2mm}\mathcal{P}(\{\mathbf{x}|f(\mathbf{x})-f(\mathbf{x}^*)<\epsilon\},a)
&\geq& 	\mathcal{P}\left(B\left(\mathbf{x}^*,\left(\epsilon/C_1\right)^\frac{1}{\alpha}\right),a\right)\nonumber\\
&\gtrsim& \epsilon^\frac{d}{\alpha}/a^d.
\end{eqnarray}

Therefore $\beta$ can not be larger than $d/\alpha$.

\section{Proof of Theorem~\ref{thm:base}}\label{app:base}


Recall that $\mathbf{c}_j$ is the center of bin $B_j$. Define
$\mathbf{c}^*=\underset{\mathbf{c}\in \{\mathbf{c}_1,\ldots,\mathbf{c}_m \}}{\arg\min}f(\mathbf{c}),$
then
\begin{eqnarray}
&&\hspace{-5mm}R=\mathbb{E}\left[\sum_{t=1}^T (f(\mathbf{X}_t)-f(\mathbf{x}^*))\right]\nonumber\\
&&\hspace{-5mm}=\sum_{j=1}^m (f(\mathbf{c}_j)-f(\mathbf{c}^*))\mathbb{E}[n_j]+T(f(\mathbf{c}^*)-f(\mathbf{x}^*)),
\label{eq:base1}
\end{eqnarray}
in which $\mathbb{E}[n_j]$ is the expectation of the number of queries in bin $B_j$.

If the bin size $a$ is small, then $f(\mathbf{c}^*)$ will be sufficiently close to $f(\mathbf{x}^*)$. To be more precise, we use Lemma \ref{lem:approx}. Note that the optimizer must belong to one of the bins. Without loss of generality, suppose $\mathbf{x}^*\in B_1$, then 
\begin{eqnarray}
f(\mathbf{c}^*)-f(\mathbf{x}^*)\leq f(\mathbf{c}(B_1))-f(\mathbf{x}^*)\leq  C_1a^\alpha,
\label{eq:center}
\end{eqnarray}
in which $\mathbf{c}(B_1)$ is the center of $B_1$.

Define 
$\Delta_j:=f(\mathbf{c}_j)-f(\mathbf{c}^*)$,
then $\mathbb{E}[n_j]$ can be bounded in the same way as the finite state space multi armed bandit problem with upper confidence bound approach \cite{auer2002finite}. Denote $g_t^*$ as the lower confidence bound estimate at $\mathbf{c}^*$. Moreover, define $t_j(s)$ and $t^*(s)$ as the time of the $s$-th query of $\mathbf{c}_j$ and $\mathbf{c}^*$, respectively. Then for any positive integer $u$,
\begin{eqnarray}
\mathbb{E}[n_j]&\leq& \sum_{t=1}^T \text{P}(\mathbf{X}_t=c_j)\nonumber\\
&=&u+\sum_{t=u+1}^T \text{P}(g_{tj}\leq g_t^*, n_{t-1,j}\geq u)\nonumber\\
&\leq & u+\sum_{t=u+1}^T \sum_{s=u+1}^t \sum_{s'=1}^t \text{P}\left(\frac{1}{s}\sum_{i=1}^{t_j(s)} \mathbf{1}(\mathbf{X}_i=\mathbf{c}_j)Y_i\right.\nonumber\\
&&\left.-\sqrt{\frac{8\ln t}{s}}\leq \frac{1}{s'}\sum_{i=1}^{t^*(s)}\mathbf{1}(\mathbf{X}_i=\mathbf{c}^*)Y_i-\sqrt{\frac{8\ln t}{s'}}\right).\nonumber
\end{eqnarray}

Let $u=\lfloor 32\ln T/\Delta_j^2\rfloor$, then
\begin{eqnarray}
\mathbb{E}[n_j]
&\leq& \frac{32\ln T}{\Delta_j^2}+\sum_{t=u+1}^T \sum_{s=u+1}^t \sum_{s'=1}^t\nonumber\\
&& \left[\text{P}\left(\frac{1}{s} \sum_{i=1}^{t_j(s)}\mathbf{1}(\mathbf{X}_i=\mathbf{c}_j)W_i\leq -\sqrt{\frac{8\ln t}{s}} \right)\right.\nonumber\\
&&\left.+\text{P}\left(\frac{1}{s'}\sum_{i=1}^{t^*(s)}\mathbf{1}(\mathbf{X}_i=\mathbf{c}^*)W_i\geq \sqrt{\frac{8\ln t}{s'}}\right) \right]\nonumber\\
&\leq & \frac{32\ln T}{\Delta_j^2}+\sum_{t=u+1}^T \sum_{s=u+1}^t \sum_{s'=1}^t \frac{2}{t^4}\nonumber\\
&\leq & \frac{32\ln T}{\Delta_j^2}+\sum_{t=1}^\infty \frac{2}{t^2}=\frac{32 \ln T}{\Delta_j^2}+\frac{\pi^2}{3}.\nonumber
\end{eqnarray}

If $\beta\leq 1$, for any $\epsilon>0$,
\begin{eqnarray}
&&\sum_{j=1}^m \Delta_j \mathbb{E}[n_j]\nonumber\\
&=&\sum_{j=1}^m \Delta_j\mathbb{E}[n_j]\mathbf{1}(\Delta_j>\epsilon)+\sum_{j=1}^m \Delta_j\mathbb{E}[n_j]\mathbf{1}(\Delta_j\leq \epsilon)\nonumber\\
&\leq &\left(\sum_{j=1}^m \Delta_j^{2-\beta}\mathbb{E}[n_j]\mathbf{1}(\Delta_j>\epsilon)\right)^\frac{1}{2-\beta} \left(\sum_{j=1}^m \mathbb{E}[n_j]\right)^\frac{1-\beta}{2-\beta}\nonumber\\
&&+\epsilon T\nonumber\\
&\leq &\left(\sum_{j=1}^m \left(\frac{32\ln T}{\Delta_j^\beta}+\frac{\pi^2}{3}\Delta_j^{2-\beta}\right)\mathbf{1}(\Delta_j>\epsilon)\right)^\frac{1}{2-\beta}T^\frac{1-\beta}{2-\beta}\nonumber\\
&&+\epsilon T,
\label{eq:1}
\end{eqnarray}
in which the second step uses H{\"o}lder's inequality. According to Assumption (a), we have $\Delta_j\leq M$. It remains to bound $\sum_{j=1}^m 32\ln T\mathbf{1}(\Delta_j>\epsilon)/\Delta_j^\beta$. Define $S_\Delta=\{j|\Delta_j\leq \Delta \}$. Then $\cup_{j\in S_\Delta} B_j$ forms a packing of $\{\mathbf{x}|f(\mathbf{x})-f(\mathbf{x}^*)\leq \Delta+f(\mathbf{c}^*)-f(\mathbf{x}^*) \}$, since for all $j\in S_\Delta$, $f(\mathbf{c}_j)-f(\mathbf{x}^*)=\Delta_j+f(\mathbf{c}^*)-f(\mathbf{x}^*)$. According to Assumption (c),
\begin{eqnarray}
|S_\Delta|&\leq& \mathcal{P}\left(\{\mathbf{x}|f(\mathbf{x})-f(\mathbf{x}^*)\leq \Delta+f(\mathbf{c}^*)-f(\mathbf{x}^*) \},a\right)\nonumber\\
&\leq &C_0\left[1+(\Delta+f(\mathbf{c}^*)-f(\mathbf{x}^*))^\beta/a^d\right].
\label{eq:Sdelta}
\end{eqnarray}
Then
\begin{eqnarray}
&&\hspace{-6mm}\sum_{j=1}^m \frac{1}{\Delta_j^\beta}\mathbf{1}(\Delta_j>\epsilon)\leq  \sum_{j=1}^m \sum_{i=0}^{\left\lfloor \frac{1}{\epsilon^\beta}\right\rfloor} \mathbf{1}\left(1/\Delta_j^\beta>i\right)\nonumber\\
&\leq & \sum_{i=0}^{\left\lfloor \frac{1}{\epsilon^\beta}\right\rfloor} \sum_{j=1}^m \mathbf{1}(\Delta_j<i^{-\frac{1}{\beta}})\nonumber\\
&\overset{(a)}{\leq} & \sum_{i=1}^{\left\lfloor \frac{1}{\epsilon^\beta}\right\rfloor} C_0\left[1+\frac{(i^{-\frac{1}{\beta}}+f(\mathbf{c}^*)-f(\mathbf{x}^*))^\beta}{a^d} \right]+m\nonumber\\
&\overset{(b)}{\leq} & C_0\frac{1}{\epsilon^\beta}+\frac{C_0}{a^d}\sum_{i=1}^{\left\lfloor \frac{1}{\epsilon^\beta}\right\rfloor} \frac{1}{i}+\frac{C_0}{a^d}\frac{1}{\epsilon^\beta}(f(\mathbf{c}^*)-f(\mathbf{x}^*))^\beta+m\nonumber\\
&\overset{(c)}{\leq} & C_0\frac{1}{\epsilon^\beta}\left(1+\frac{C_1^\beta a^{\alpha \beta}}{a^d}\right)+\frac{C_0}{a^d}\left(\ln \frac{1}{\epsilon^\beta}+1\right)+m.
\end{eqnarray}
(a) comes from \eqref{eq:Sdelta}. Moreover, for $i=0$, we just bound $\sum_{j=1}^m \mathbf{1}(\Delta_j<i^{-\frac{1}{\beta}})$ by $m$. For (b), recall that we have assumed $\beta\leq 1$ in the statement of Theorem 1, hence $(i^{-\frac{1}{\beta}}+f(\mathbf{c}^*)-f(\mathbf{x}^*))^\beta\leq i^{-1}+(f(\mathbf{c}^*-f(\mathbf{x}^*)))^\beta$. (c) uses \eqref{eq:center}.

Let $\epsilon=a^\alpha$. Then
\begin{eqnarray}
\sum_{j=1}^m \frac{1}{\Delta_j^\beta}\mathbf{1}(\Delta_j>\epsilon)\leq \frac{C_2}{a^d}\ln \frac{1}{a}+m
\end{eqnarray}
for some constant $C_2$. Since $\mathcal{S}$ is compact, we have $m\sim 1/a^d$. From \eqref{eq:1}, we have
\begin{eqnarray}
\sum_{j=1}^m \Delta_j\mathbb{E}[n_j]&\lesssim& \left(\frac{1}{a^d}\ln \frac{1}{a}\right)^\frac{1}{2-\beta}+\epsilon T\nonumber\\
&\sim & \left(\frac{1}{a^d}\ln \frac{1}{a}\right)^\frac{1}{2-\beta}+Ta^\alpha.
\end{eqnarray}
From \eqref{eq:base1},
\begin{eqnarray}
R&=&\sum_{j=1}^m \Delta_j\mathbb{E}[n_j]+T(f(\mathbf{c}^*)-f(\mathbf{x}^*))\nonumber\\
&\lesssim & a^{-\frac{d}{2-\beta}}\left(-\ln a\right)^\frac{1}{2-\beta}+Ta^\alpha.
\end{eqnarray}

Let
$a\sim T^{-\frac{1}{d+2\alpha-\alpha\beta}},$
then
$R/T\lesssim T^{-\frac{\alpha}{d+2\alpha-\alpha \beta}}(\ln T)^\frac{1}{2-\beta}.$
The proof is complete.

\section{Proof of Theorem \ref{thm:adaptive}}\label{app:adaptive}
Let $g_t(\mathbf{x}^*)$ be the lower confidence bound value of the bin containing $\mathbf{x}^*$. Denote such bin as $B^*$. Moreover, let $n_t(B)$ be the number of queries in $B$ until time $t$, and $\mathbf{c}(B)$ be the center of $B$. Furthermore, let $I_s(B)$ be the set of first $s$ queries in $B$ according to the time order. Then the following  two lemmas hold:
\begin{lem}\label{lem:largeg}
	There exist two constants $C_3$ and $C_4$, such that if $\mu>C_1$, in which $C_1$ is the constant in Lemma \ref{lem:approx}, then
	\begin{eqnarray}
	\text{P}(g_{t+1}(\mathbf{x}^*)>f(\mathbf{x}^*))\leq C_4t\ln t\exp[-C_3\ln^2 t].
	\end{eqnarray}
\end{lem}
\begin{lem}\label{lem:prob}
	If a bin $B$ has been split $k$ times, $k> k_c$ for some constant $k_c$, and 
	\begin{eqnarray}
	f(\mathbf{c}(B'))-f(\mathbf{x}^*)\geq 3\times 2^{-\alpha (k-1)}\max\{ \ln t, \mu a_0^\alpha \},
	\label{eq:probcond}
	\end{eqnarray} 
	in which $B'$ is the parent of $B$, then the probability that the $(t+1)$th query falls in $B$ is bounded by
	$\text{P}(t+1\in I(B))\leq \phi_1(t_k),$
	for some function $\phi_1$ that decays faster than any polynomial, in which 
	$t_k=\left\lceil \frac{2^{2\alpha k}-1}{2^{2\alpha}-1}\right\rceil.$
\end{lem}
Lemma \ref{lem:largeg} validates the construction of lower confidence bound by showing that with high probability, $g_t(\mathbf{x}^*)$ does not exceed $f(\mathbf{x}^*)$, while Lemma \ref{lem:prob} shows that the probability of a query falling in each bin $B$ will be low for large $k$. Based on Lemma \ref{lem:prob}, we can then bound the number of queries in each bin. In particular, we show the following lemma.
\begin{lem}\label{lem:n}
	If $k>k_c$, and the following two condition holds:
	\begin{eqnarray}
	f(\mathbf{c}(B')-f(\mathbf{x}^*)&\geq& 3\times 2^{-\alpha(k-1)}\max\{ \ln T, \mu a_0^\alpha \},\nonumber\\
	f(\mathbf{c}(B))-f(\mathbf{x}^*)&\geq& 3\mu a_k^\alpha,
	\label{eq:ncond}
	\end{eqnarray}
	then
	\begin{eqnarray}
	\mathbb{E}[n(B)]\leq \frac{\ln^2 T}{(f(\mathbf{c}(B))-f(\mathbf{x}^*))^2} \phi_2(t_k),
	\end{eqnarray}
	for some function $\phi_2$ that decays faster than any polynomial.
\end{lem}

With these lemmas, we then bound the expected cumulative regret. Here we use the following notation. Let $B_{kj}, k=0,1,\ldots, j=1,\ldots,m_k$ be the $j$-th bin among all bins that have been divided $k$ times. Then according to the assumption that $S$ is compact, there exists a constant $C_s$, such that
\begin{eqnarray}
m_k\leq C_s/a^d=C_s2^{kd}/a^d.
\label{eq:mk}
\end{eqnarray}
Define
$R_k:=\mathbb{E}\left[\sum_{t=1}^T \sum_{j=1}^{m_k}\mathbf{1}(t\in B_{kj})(f(\mathbf{X}_t)-f(\mathbf{x}^*))\right],$
then
$R=\sum_{k=0}^\infty R_k.$
To bound $R$, we bound each $R_k$ for the following three cases separately. 

\textbf{Case 1: $k\leq k_c$}. In this case, Lemmas \ref{lem:prob} and \ref{lem:n} do not hold. We provide a simple upper bound to $R_k$. According to the adaptive partition rule,
$n(B_{kj})\leq \left\lceil 2^{2\alpha k}\right\rceil\leq \left\lceil 2^{2\alpha k_c}\right\rceil.$
Therefore
\begin{eqnarray}
R_k&\leq & M\sum_{j=1}^{m_k}\mathbb{E}\left[\sum_{t=1}^T \mathbf{1}(t\in I(B_{kj}))\right]\nonumber\\
&\leq &Mm_k\lceil 2^{2\alpha k_c}\rceil 
\leq  \frac{MC_s}{a_0^d}2^{k_cd}\lceil 2^{2\alpha k_c}\rceil.
\end{eqnarray}

\textbf{Case 2: $k_c<k\leq \log_2 T/(d+2\alpha-\alpha \beta)$}. In this case, we define
\begin{eqnarray}
R_{kj}=\mathbb{E}\left[\sum_{t=1}^T \mathbf{1}(t\in I(B_{kj}))(f(\mathbf{X}_t)-f(\mathbf{x}^*))\right],
\end{eqnarray}
for $j=1,\ldots,m_k$. then
$R_k=\sum_{j=1}^{m_k} R_{kj}.$

We define the following two sets:
\begin{eqnarray}
A_k&=&\{j|f(\mathbf{c}((B_{kj}'))-f(\mathbf{x}^*)\geq 3\times 2^{-\alpha (k-1)}\max\{\ln T,\nonumber\\
&& \mu a_0^\alpha\}, f(\mathbf{c}((B_{kj}))-f(\mathbf{x}^*)\geq 3\mu a_k^\alpha \},\nonumber\\
B_k&=&\{1,\ldots, m_k \}\setminus A_k.
\label{eq:akbk}
\end{eqnarray}
Then for all $j\in A_k$, we have
\begin{eqnarray}
R_{kj}
&=& \mathbb{E}\left[\sum_{t=1, t\in I(B_{kj})}^T\mathbb{E}[f(\mathbf{X}_t)-f(\mathbf{x}^*)|\mathbf{X}_t\in I(B_{kj})]\right]\nonumber\\
&\overset{(a)}{\leq}& \mathbb{E}\left[\sum_{t=1, t\in I(B_{kj})}^T (f(\mathbf{c}(B_{kj}))-f(\mathbf{x}^*)+Ma_k^\alpha)\right]\nonumber\\
&\overset{(b)}{\leq}& \frac{(\ln^2 T)\phi_2(t_k)}{(f(\mathbf{c}(B_{kj}))-f(\mathbf{x}^*))^2}(f(\mathbf{c}(B_{kj}))-f(\mathbf{x}^*)+Ma_k^\alpha)\nonumber\\
&\lesssim & (\ln^2 T) 2^{k\alpha}\phi_2(t_k),
\label{eq:Rkj1}
\end{eqnarray}
in which $A\lesssim B$ means that $A\leq CB$ for some constant $C$ that depends only on $a_0, \alpha, M, \mu, C_0, C_s, \beta, d$. (a) uses Assumption (b), and (b) comes from Lemma \ref{lem:n}.

Now we bound $R_{kj}$ for $j\in B_k$ via the following lemma.
\begin{lem}\label{lem:Ef}
	For all $j\in B_k$,
	\begin{eqnarray}
	\mathbb{E}[f(\mathbf{X}_t)-f(\mathbf{x}^*)|t\in I(B_{kj})]
	\leq 2^{d+2-\alpha (k-1)}(\ln T+\mu a_0^\alpha).\nonumber
	\end{eqnarray}
\end{lem}
Based on Lemma \ref{lem:Ef}, we have
\begin{eqnarray}
R_{kj}
&=&\mathbb{E}\left[\sum_{t=1}^T \mathbf{1}(t\in I(B_{kj}))\mathbb{E}[f(\mathbf{X}_t)-f(\mathbf{x}^*)|t\in I(B_{kj})]\right]\nonumber\\
&\leq &2^{d+2-\alpha(k-1)}(\ln T+\mu a_0^\alpha)\mathbb{E}[n(B_{kj})]\nonumber\\
&\leq & 2^{d+2-\alpha(k-1)}(\ln T+\mu a_0^\alpha)(2^{2\alpha k}+1)
\lesssim  2^{\alpha k}\ln T.\nonumber
\end{eqnarray}
Moreover, the size of $B_k$ is bounded by
\begin{eqnarray}
&&\hspace{-5mm}|B_k|\leq \sum_{j=1}^{m_k}\mathbf{1}(f(\mathbf{c}(B_{kj}))-f(\mathbf{x}^*)<3\mu a_0 2^{-k\alpha} \text{ or }\nonumber\\
&& f(\mathbf{c}(B_{kj}'))-f(\mathbf{x}^*)<3\times 2^{-\alpha(k-1)}\max\{\ln T, \mu a_0^\alpha \} )\nonumber\\
&\leq & \sum_{j=1}^{m_k} \mathbf{1}(f(\mathbf{c}(B_{kj}))-f(\mathbf{x}^*)<3\mu a_0 2^{-k\alpha})\nonumber\\
&& +2^d \sum_{j=1}^{m_k}\mathbf{1}(f(\mathbf{c}(B_{k-1,j}))-f(\mathbf{x}^*)<3\times 2^{-\alpha(k-1)}\nonumber\\
&&\max\{\ln T, \mu a_0^\alpha \}\nonumber\\
&\leq & C_0\left[1+\frac{(3\mu a_0 2^{-k\alpha})^\beta}{a_0 2^{-kd}}\right]\nonumber\\
&&+2^d C_0\left[1+\frac{(3(\ln T+\mu a_0^\alpha)2^{-\alpha (k-1)})^\beta}{a_0^d 2^{-(k-1)\alpha d}} \right]\nonumber\\
&\lesssim & 2^{k(d-\alpha \beta)}\ln^\beta T.
\end{eqnarray}
Hence,
\begin{eqnarray}
\hspace{-5mm}R_k&=&\sum_{j=1}^{m_k} R_{kj}=\sum_{j\in A_k} R_{kj}+\sum_{j\in B_k} R_{kj}\nonumber\\
&\lesssim & (\ln^2 T) 2^{k\alpha}\phi_2(t_k)m_k+2^{k(d-\alpha \beta)}2^{\alpha k}\ln^{\beta+1}T\nonumber\\
&\lesssim & 2^{k(d+\alpha-\alpha \beta)}\ln^{\beta+1}T+2^{k(\alpha+d)}\phi_2(t_k)\ln^2 T.\nonumber
\end{eqnarray}

\textbf{Case 3: $k>\log_2 T/(d+2\alpha-\alpha \beta)$}. In this case, for all $j\in A_k$, \eqref{eq:Rkj1} still holds. For $j\in B_k$, 
$R_{kj}\lesssim 2^{-\alpha k}\mathbb{E}[n(B_{kj})]\ln T.$
Hence
\begin{eqnarray}
R_k&=&\sum_{j\in A_k} R_{kj}+\sum_{j\in B_k} R_{kj}\nonumber\\
&\lesssim & 2^{k\alpha}\phi_2(t_k)m_k\ln^2 T+\sum_{j\in B_k} 2^{-\alpha k}\mathbb{E}[n(B_{kj})]\ln T\nonumber\\
&\lesssim & 2^{-\alpha k}T\ln T+2^{k(\alpha+d)}\phi_2(t_k)\ln^2 T.
\end{eqnarray}
Combining these three parts, we have
\begin{eqnarray}
&&R=\sum_{k=0}^\infty R_k\nonumber\\
&\lesssim& (k_c+1)\frac{MC_s}{a_0^d}2^{k_c d}\lceil 2^{2\alpha k_c}\rceil\sum_{k=k_c+1}^\infty 2^{k(\alpha+d)}\phi_2(t_k)\ln^2 T\nonumber\\
&&+\sum_{k=k_c+1}^{\left\lfloor \frac{\log_2 T}{d+2\alpha-\alpha \beta}\right\rfloor} 2^{k(d+2\alpha-\alpha \beta)}\ln^{\beta+1}T\nonumber\\ &&+\sum_{k=\left\lfloor \frac{\log_2 T}{d+2\alpha-\alpha \beta}\right\rfloor+1}^\infty 2^{-\alpha k}T\ln T.
\end{eqnarray}
Recall that $t_k=\lceil (2^{2\alpha k}-1)/(2^{2\alpha}-1)\rceil$, and $\phi_2$ decays faster than any polynomial, the first and the second term are both upper bounded by constants. Hence
\begin{eqnarray}
R &\lesssim & \sum_{k=k_c+1}^{\left\lfloor \frac{\log_2 T}{d+2\alpha-\alpha \beta}\right\rfloor} 2^{\log_2 T\frac{d+\alpha-\alpha \beta}{d+2\alpha-\alpha \beta}}\ln^{\beta+1}T\nonumber\\
&&+T\ln T \frac{2^{-\alpha \frac{\log_2 T}{d+2\alpha-\alpha \beta}}}{1-2^{-\alpha}}
\lesssim  T^{\frac{d+\alpha-\alpha \beta}{d+2\alpha-\alpha \beta}}\ln^{\beta+1}T,\nonumber
\end{eqnarray}
i.e.
$R/T\lesssim T^{-\frac{\alpha}{d+2\alpha-\alpha \beta}}\ln^{\beta+1}T.$
The proof is complete.
\subsection{Proof of Lemma \ref{lem:largeg}}\label{sec:largeg}
Define $B_k^*$ as the bin that contains $\mathbf{x}^*$ and has been split for $k$ times, and $\mathbf{c}(B_k^*)$ as its center. Then
\begin{eqnarray}
&&\hspace{-5mm}\text{P}(g_{t+1}(\mathbf{x}^*)>f(\mathbf{x}^*))\nonumber\\
&\leq &\text{P}\left(\exists k, \frac{1}{n_t(B_k^*)}\sum_{i=1}^t \mathbf{1}(i\in I(B_k^*))Y_i\right.\nonumber\\
&&\left.-\mu a^\alpha -n_t^{-\frac{1}{2}}(B_k^*)\ln t>f(\mathbf{x}^*)\right)\nonumber\\
&\overset{(a)}{\leq}&\sum_{k=0}^{k_{max}} \text{P}\left(\frac{1}{n_t(B_k^*)}\sum_{i=1}^t \mathbf{1}(i\in I(B_k^*)) (f(\mathbf{X}_i)+W_i)\right.\nonumber\\
&&\left.-\mu a^\alpha-n_t^{-\frac{1}{2}}(B_k^*)\ln t>f(\mathbf{x}^*)\right)\nonumber\\
&\leq & \sum_{k=0}^{k_{max}} \text{P}\left(f(\mathbf{c}(B_k^*))+\frac{1}{n_t(B_k^*)}\sum_{i=1}^t \mathbf{1}(i\in I(B_k^*)) (f(\mathbf{X}_i)\right.\nonumber\\
&&\left.-f(\mathbf{c}(B_k^*))+W_i)-\mu a^\alpha-n_t^{-\frac{1}{2}}(B_k^*)\ln t>f(\mathbf{x}^*) \right)\nonumber\\
&\overset{(b)}{\leq}& \sum_{k=0}^{k_{max}} \text{P}\left(C_1a_k^\alpha+\frac{1}{n_t(B_k^*)}\sum_{i=1}^t \mathbf{1}(i\in I(B_k^*)) (f(\mathbf{X}_i)\right.\nonumber\\
&&\left.-f(\mathbf{c}(B_k^*))+W_i)-\mu a^\alpha+n_t^{-\frac{1}{2}}(B_k^*)\ln t>0\right)\nonumber\\
&\overset{(c)}{\leq}& \sum_{k=0}^{k_{max}}\sum_{s=1}^t \text{P}\left(\frac{1}{s}\sum_{i=1}^t \mathbf{1}(i\in I_s(B))(f(\mathbf{X}_i)-f(\mathbf{c}(B_k^*)\right.\nonumber\\
&&\left.+W_i)>s^{-\frac{1}{2}}\ln t\right)\nonumber\\
&\overset{(d)}{\leq} &\sum_{k=0}^{k_{max}}\sum_{s=1}^t \exp\left[-C_3\ln^2 t\right]
\leq C_4t\ln t \exp[-C_3\ln^2 t],\nonumber
\end{eqnarray}
for some constants $C_3$ and $C_4$. 

In (a), note that if $B_k^*$ is obtained after $k$ splittings, then its ancestors must be full, however, the total number of samples should be no more than $t$. Therefore
$\sum_{l=1}^{k-1}\lceil 2^{2\alpha l}\rceil \leq t$,
which yields
$k\leq \log_2 t/(2\alpha)$.
We define $k_{max}=\lfloor \log_2 t/(2\alpha)\rfloor$. (b) uses Lemma \ref{lem:approx}. For (c), consider that $n_t(B_k^*)$ itself is a random variable, we can not directly use concentration inequalities to give the bound. However, consider that $n_t(B_k^*)\in \{1,\ldots, t\}$, we convert the bound to a union bound over $s=1,\ldots,t$ in (c). (d) uses Hoeffding's inequality for sub-Gaussian random variables. From Assumption \ref{ass}(a),
$|f(\mathbf{X}_i)-f(\mathbf{c}(B_k^*))|\leq M$.
According to Hoeffding's lemma, $f(\mathbf{X}_i)-f(\mathbf{c}(B_k^*))$ is subGaussian with parameter $M^2/2$. Note that $W_i$ is independent with $f(\mathbf{X}_i)-f(\mathbf{c}(B_k^*))$, and $W_i$ follows standard Gaussian distribution, therefore $f(\mathbf{X}_i)-f(\mathbf{c}(B_k^*))+W_i$ follows a subGaussian distribution with parameter $M^2/2+1$. Use Hoeffding's inequality,
\begin{eqnarray}
&&\hspace{-12mm}\text{P}\left(\frac{1}{s}\sum_{i=1}^t \mathbf{1}(i\in I_s(B))(f(\mathbf{X}_i)-f(\mathbf{c}(B_k^*)+W_i)>s^{-\frac{1}{2}}\ln t\right)\nonumber\\
&\leq& \exp\left[-\frac{\ln^2 t}{M^2+2}\right].
\end{eqnarray}
Define $C_3=1/(M^2+2)$, then (d) holds.

\subsection{Proof of Lemma \ref{lem:prob}}\label{sec:prob}
Since $B$ is a bin obtained after $k$ splittings, the total number of queries is at least
$t_{min}=1+\lceil 2^{2\alpha}\rceil +\ldots + \lceil2^{2\alpha (k-1)}\rceil$
times. Recall that
$t_k=\left\lceil \frac{2^{2\alpha k}-1}{2^{2\alpha}-1}\right\rceil,$
thus $t_{min}\geq t_k$. Define
$N_k=\lceil 2^{2\alpha k}\rceil$
as the threshold of the number of samples in bins that has been splitted $k$ times. Then
\begin{eqnarray}
&&\text{P}(t+1\in I(B))\nonumber\\
&\leq& \text{P}\left(\exists s, t_k\leq s\leq t, g_{s+1}(B')<g_{s+1}(\mathbf{x}^*),\right.\nonumber\\
&&\left. n_s(B')=N_{k-1}\right)\nonumber\\
&\leq & \text{P}(\exists s, t_k\leq s\leq t, g_{s+1}(\mathbf{x}^*)<f(\mathbf{x}^*))\nonumber\\
&&+\text{P}\left(\exists s, t_k\leq s\leq t, \frac{1}{N_{k-1}}\sum_{i=1}^s \mathbf{1}(i\in I(B'))Y_i\right.\nonumber\\
&&\left.-\mu a_{k-1}^\alpha -N_k^{-\frac{1}{2}}\ln t<f(\mathbf{x}^*), n_s(B')=N_{k-1}\right)\nonumber\\
&\overset{(a)}{\leq} & \sum_{s=t_k}^t \text{P}(g_{s+1}(\mathbf{x}^*)<f(\mathbf{x}^*))\nonumber\\
&&+\sum_{s=t_k}^t \text{P}\left(\frac{1}{N_{k-1}}\sum_{i=1}^s\mathbf{1}(i\in I(B'))(f(\mathbf{X}_i)-f(\mathbf{c}(B'))\right.\nonumber\\
&&\left.+W_i)<-\frac{1}{3}(f(\mathbf{c}(B'))-f(\mathbf{x}^*)), n_s(B')=N_{k-1}\right)\nonumber\\
&\overset{(b)}{\leq} & \sum_{s=t_k}^t \left[C_4 s\ln s \exp(-C_3\ln^2 s)+\exp\left(-\frac{\ln^2 t}{M^2 +2}\right)\right]\nonumber\\
&=&\sum_{s=t_k}^t (C_1s\ln s\exp(-C_2\ln^2 s)+\exp(-C_2\ln^2 t))\leq  \phi_1(t_k),\nonumber
\end{eqnarray}
for some function $\phi_1$ that decays faster than any polynomial.

In (a), we use \eqref{eq:probcond}, which yields 
\begin{eqnarray}
\mu a_{k-1}^\alpha &<&\frac{1}{3}(f(c(B'))-f(\mathbf{x}^*)),\\
N_k^{-\frac{1}{2}}\ln t&<&\frac{1}{3}(f(c(B'))-f(\mathbf{x}^*)),
\end{eqnarray}
hence (a) holds. (b) uses Lemma \ref{lem:largeg} and Hoeffding's inequality. The proof of Lemma \ref{lem:prob} is complete.

\subsection{Proof of Lemma \ref{lem:n}}\label{sec:n}
Let 
\begin{eqnarray}
u=\left\lceil 9\ln^2 T/(f(c(B))-f(\mathbf{x}^*))^2\right\rceil,
\label{eq:u}
\end{eqnarray}
and define $t_B(s)$ as the time such that $B$ contains exactly $s$ queries, i.e.
$t_B(s)=\inf\{t'|n_{t'}(B)=s \}$,
then
\begin{eqnarray}
&&\hspace{-4mm}\mathbb{E}[n(B)]=\mathbb{E}\left[\sum_{t=1}^T \mathbf{1}(t\in I(B))\right]\nonumber\\
&\overset{(a)}{=}& \mathbb{E}\left[\sum_{t=t_k+1}^T \mathbf{1}(t\in I(B))\right]\nonumber\\
&=&\mathbb{E}\left[\sum_{t=t_k+1}^T\mathbf{1}(t\in I(B)), n_{t-1}(B)\leq u\right]\nonumber\\
&&+\mathbb{E}\left[\sum_{t=t_k+1}^T \mathbf{1}(t\in I(B), n_{t-1}(B)\geq u\right]\nonumber\\
&\overset{(b)}{\leq} & u\phi_1(t_k) \nonumber\\
&&+\sum_{t=t_k+u+1}^T\text{P}(g_t(B)\leq g_t(B^*),n_{t-1}(B)\geq u)\nonumber\\
&=& u\phi_1(t_k)+\sum_{t=t_k+u+1}^T \left[\text{P}\left(\frac{1}{n_{t-1}(B)}\sum_{i=1}^{t-1} \mathbf{1}(i\in I(B))Y_i\right.\right.\nonumber\\
&&\left.-\mu a_k^\alpha+n_{t-1}^{-\frac{1}{2}}(B)\ln t<f(\mathbf{x}^*),n_{t-1}(B)\geq u\right)\nonumber\\
&& \left.+\text{P}(g_t(\mathbf{x}^*)>f(\mathbf{x}^*)) \right]\nonumber\\
&\overset{(c)}{\leq}& u\phi_1(t_k) +\sum_{t=t_k+u+1}^T \left[\sum_{s=u+1}^t \text{P}\left(\frac{1}{s}\sum_{i=1}^{t_B(s)} \mathbf{1}(i\in I(B))\right.\right.\nonumber\\
&&(f(\mathbf{X}_i)-f(\mathbf{c}(B))+W_i)<\nonumber\\
&&\left.\left.-\frac{1}{3}(f(\mathbf{c}(B))-f(\mathbf{x}^*)\right) +\text{P}(g_t(\mathbf{x}^*)>f(\mathbf{x}^*)) \right]\nonumber\\
&\leq & u\phi_1(t_k) +\sum_{t=t_k+u+1}^T \left[\sum_{s=u+1}^t \text{P}\left(\frac{1}{s}\sum_{i=1}^{t_B(s)} \mathbf{1}(i\in I(B))\right.\right.\nonumber\\
&&\left.(f(\mathbf{X}_i)-f(\mathbf{c}(B))+W_i)<-s^{-\frac{1}{2}}\ln t\right)\nonumber\\
&&\left. +\text{P}(g_t(\mathbf{x}^*)>f(\mathbf{x}^*)) \right]\nonumber\\
&\leq &u\phi_1(t_k)+\sum_{t=t_k+u+1}^T\left[t\exp\left(-\frac{\ln^2 t}{M^2+2}\right)\right.\nonumber\\
&&\left.+\text{P}(g_t(\mathbf{x}^*)>f(\mathbf{x}^*))\right]\nonumber\\
&\leq &\left(\frac{9\ln^2 T}{(f(\mathbf{c}(B))-f(\mathbf{x}^*))^2}+1\right)\phi_1(t_k)\nonumber\\
&&+\sum_{t=t_k+u+1}^T(C_3 t\ln t+t)\exp[-C_2\ln^2 t]\nonumber\\
&\leq & \frac{\ln^2 T}{(f(\mathbf{c}(B))-f(\mathbf{x}^*))^2}\phi_2(t_k),
\end{eqnarray}
for some function $\phi_2$ that decays faster than any polynomial.

For (a), note that when $t\leq t_k$, $\text{P}(t\in I(B))=0$. Since $B$ has been split for $k$ times, the ancestors must be full of queries, therefore if a new query lies in $B$, $t>t_k$ must hold. (b) uses Lemma \ref{lem:prob}. In (c), recall the condition \eqref{eq:ncond}, we have
$\mu a_k^\alpha \leq \frac{1}{3}(f(\mathbf{c}(B))-f(\mathbf{x}^*)),$
and from \eqref{eq:u},
\begin{eqnarray}
n_{t-1}^{-\frac{1}{2}}\ln t\leq u^{-\frac{1}{2}}\ln t\leq \frac{1}{3} (f(\mathbf{c}(B))-f(\mathbf{x}^*)),
\end{eqnarray}
hence (c) holds.

\subsection{Proof of Lemma \ref{lem:Ef}}\label{sec:Ef}
According to the definition of $B_k$ in \eqref{eq:akbk}, for all $j\in B_k$, at least one of the following two conditions holds:
\begin{eqnarray}
&&\hspace{-7mm}f(\mathbf{c}(B_{kj}')-f(\mathbf{x}^*)<3\times 2^{-\alpha(k-1)}\max\{\ln T, \mu a_0^\alpha \},
\label{eq:cond1}\\
&&\hspace{-7mm}f(\mathbf{c}(B_{kj}))-f(\mathbf{x}^*)<3\mu a_k^\alpha.
\label{eq:cond2}
\end{eqnarray}

If \eqref{eq:cond2} holds, then
\begin{eqnarray}
&&\hspace{-5mm}\mathbb{E}[f(\mathbf{X}_t)-f(\mathbf{x}^*)|t\in I(B_{kj})]\nonumber\\
&&\hspace{-5mm}=\mathbb{E}[f(\mathbf{X}_t)-f(\mathbf{c}(B_{kj}))|t\in I(B_{kj})]+f(\mathbf{c}(B_{kj}))-f(\mathbf{x}^*)\nonumber\\
&&\hspace{-5mm}\leq  Ma_k^\alpha +3\mu a_k^\alpha\leq  \mu a_0 2^{2-k\alpha}.
\label{eq:Ef1}
\end{eqnarray}

If \eqref{eq:cond1} holds, then according to Assumption (b),
\begin{eqnarray}
&&\mathbb{E}[f(\mathbf{X}_t)-f(\mathbf{x}^*)|t\in I(B_{kj})]\nonumber\\
&=&\frac{1}{V(B_{kj})}\int_{B_{kj}}(f(\mathbf{x})-f(\mathbf{x}^*))d\mathbf{x}\nonumber\\
&\leq &\frac{1}{V(B_{kj})}\int_{B_{kj}'} (f(\mathbf{x})-f(\mathbf{x}^*))d\mathbf{x}\nonumber\\
&= & \frac{2^d}{V(B_{kj}')}\int_{B_{kj}'}(f(\mathbf{x})-f(\mathbf{x}^*))d\mathbf{x}\nonumber\\
&\leq & 2^d(Ma_{k-1}^\alpha)+f(\mathbf{c}(B_{kj}'))-f(\mathbf{x}^*)\nonumber\\
&\leq & 2^d (Ma_{k-1}^\alpha+3\times 2^{-\alpha (k-1)}\max\{\ln T,\mu a_0^\alpha \} )\nonumber\\
&\leq & 4\times 2^{d-\alpha(k-1)}\max\{\ln T, \mu a_0^\alpha \}.
\label{eq:Ef2}
\end{eqnarray}
Combine \eqref{eq:Ef1} and \eqref{eq:Ef2}, Lemma \ref{lem:Ef} holds.

\small \bibliography{macros,optimization}
\bibliographystyle{ieeetran}

\end{document}